\documentclass{article}

\usepackage[numbers]{natbib}

     \usepackage[preprint]{neurips_2022}

\usepackage[utf8]{inputenc} %
\usepackage[T1]{fontenc}    %
\usepackage{hyperref}       %
\usepackage{url}            %
\usepackage{booktabs}       %
\usepackage{amsfonts}       %
\usepackage{nicefrac}       %
\usepackage{microtype}      %
\usepackage{xcolor}         %

\usepackage{amsmath}
\usepackage{amssymb}
\usepackage{mathtools}
\usepackage{amsthm}

\usepackage{thm-restate}

\usepackage[capitalize,noabbrev]{cleveref}

\usepackage[T1]{fontenc}    %
\usepackage{hyperref}       %
\usepackage{url}            %
\usepackage{booktabs}       %
\usepackage{amsfonts}       %
\usepackage{nicefrac}       %
\usepackage{microtype}      %
\usepackage{xcolor}         %
\usepackage{amsmath,amsfonts,bm}
\usepackage{amssymb}
\usepackage{graphicx}
\usepackage{xspace}
\usepackage{mathtools}
\usepackage{amsthm}
\usepackage{amssymb}

\usepackage{microtype}
\usepackage{comment}
\usepackage{epsf}
\usepackage{fancyhdr}
\usepackage{graphics}
\usepackage{graphicx}
\usepackage{mathabx}
\usepackage{psfrag}
\usepackage{savesym}
\usepackage[T1]{fontenc}
\usepackage{color}
\usepackage{amsthm}
\usepackage{amsfonts}
\usepackage{amsmath}
\usepackage{amssymb}
\usepackage{mathrsfs}
\usepackage{bm}
\usepackage{accents}
\usepackage{listings}
\usepackage{caption}
\usepackage{mleftright}
\usepackage{subcaption}
\usepackage[titletoc,toc,title]{appendix}
\usepackage{enumerate}
\usepackage{verbatim} %
\usepackage{csquotes} %
\usepackage{upquote} %
\usepackage[bottom]{footmisc} %
\usepackage{thmtools}
\usepackage{thm-restate}

\usepackage{algorithm}
\usepackage{algorithmicx}

\newtheorem{theorem}{Theorem}[section]
\newtheorem{lemma}[theorem]{Lemma}

\newtheorem{corollary}[theorem]{Corollary}

\newtheorem{assumption}[theorem]{Assumption}

\newcommand{\norm}[1]{\left\lVert #1 \right\rVert}

\newcommand{\N}{\mathbb{N}}

\newcommand{\E}{\mathbb{E}}

\renewcommand{\Pr}{\mathbb{P}}
\newcommand{\lv}{\lVert}
\newcommand{\rv}{\rVert}

\renewcommand{\epsilon}{\varepsilon}
\renewcommand{\ln}{\log}
\DeclareSymbolFont{extraup}{U}{zavm}{m}{n}
\DeclareMathSymbol{\varheart}{\mathalpha}{extraup}{86}
\DeclareMathSymbol{\vardiamond}{\mathalpha}{extraup}{87}

\DeclareMathOperator*{\argmax}{arg\,max}
\DeclareMathOperator*{\argmin}{arg\,min}

\newcommand{\thetastar}{\btheta^{\star}}
\newcommand{\M}{\mathbf{M}}
\newcommand{\Mstar}{\mathbf{M}_\star}
\newcommand{\K}{\mathbf{K}}
\newcommand{\A}{\mathbf{A}}

\newcommand{\x}{\mathbf{x}}

\newcommand{\ts}{\tilde{s}}

\newcommand{\btheta}{\bm{\theta}}

\renewcommand{\epsilon}{\varepsilon}
\renewcommand{\ln}{\log}

\renewcommand{\Pr}{\mathbb{P}}
\renewcommand{\epsilon}{\varepsilon}
\renewcommand{\ln}{\log}

\makeatletter
\newcommand{\pushright}[1]{\ifmeasuring@#1\else\omit\hfill$\displaystyle#1$\fi\ignorespaces}
\newcommand{\pushleft}[1]{\ifmeasuring@#1\else\omit$\displaystyle#1$\hfill\fi\ignorespaces}
\makeatother

\title{Joint Representation Training in Sequential Tasks with Shared Structure}

\author{%
  Aldo Pacchiano \\
  Microsoft Research, NYC\\
  \texttt{apacchiano@microsoft.com} \\
  \And
  Ofir Nachum\\
  Google Research\\
  \texttt{ofirnachum@google.com} \\
  \AND
  Nilesh Tripuraneni\\
  UC Berkeley\\
  \texttt{nilesh\_tripuraneni@berkeley.edu}\\
   \And
  Peter Bartlett\\
  UC Berkeley\\
  \texttt{peter@berkeley.edu}\\
}

\begin{document}

\maketitle

\begin{abstract}
 Classical theory in reinforcement learning (RL) predominantly focuses on the single task setting, where an agent learns to solve a task through trial-and-error experience, given access to data only from that task. However, many recent empirical works have demonstrated the significant practical benefits of leveraging a joint representation trained across multiple, related tasks. In this work we theoretically analyze such a setting, formalizing the concept of \emph{task relatedness} as a shared state-action representation that admits linear dynamics in all the tasks.
We introduce the \textsf{Shared-MatrixRL} algorithm for the setting of Multitask \textsf{MatrixRL}~\cite{yang2020reinforcement}. In the presence of $P$ episodic tasks of dimension $d$ sharing a joint $r \ll d$ low-dimensional representation, we show the regret on the the $P$ tasks can be improved from $O(PHd\sqrt{NH})$ to $O((Hd\sqrt{rP} + HP\sqrt{rd})\sqrt{NH})$ over $N$ episodes of horizon $H$. These gains coincide with those observed in other linear models in contextual bandits and RL~\cite{yang2020impact,hu2021near}. In contrast with previous work that have studied multi task RL in other function approximation models, we show that in the presence of bilinear optimization oracle and finite state action spaces there exists a computationally efficient algorithm for multitask \textsf{MatrixRL} via a reduction to quadratic programming. We also develop a simple technique to shave off a $\sqrt{H}$ factor from the regret upper bounds of some episodic linear problems.

\end{abstract}

\section{Introduction}

Reinforcement learning (RL) is about learning via doing -- learning to solve a sequential decision-making task where the only information about the task is obtained via trial-and-error. 
Accordingly, the underlying assumptions made in RL are typically minimal. Beyond what can be learned from trial-and-error experience, the learner's structural prior on the underlying task is commonly restricted to a small set of Markov assumptions~\citep{puterman1990markov}: namely, that the task is of a sequential nature, with the task reward and state transition dynamics at each step determined by an (unknown) Markov process.

The simplicity of this setting, which forms the basis of a rich and diverse literature~\citep{bertsekas2019reinforcement,sutton1992introduction}, stands in contrast to the complexity of many real-world settings, where one has access to data from multiple, related tasks. In these situations, experience from one task can often be leveraged to accelerate learning in another. For example, when humans are confronted with learning a new video game, we naturally draw on previous experience and knowledge from playing other games, even if the dynamics and rewards across the games are not the same.

In line with this intuition, there exist a number of empirical works which demonstrate how experience can be gathered from multiple tasks to accelerate RL over learning these tasks in isolation. 
For example in robotics~\citep{yu2020meta,kalashnikov2021mt}, such approaches are key to avoiding an expensive blow-up in the sample complexity of required, real-world interactions.
A popular paradigm to jointly use experience from multiple tasks is by way of learning a \emph{shared low-dimensional representation}. 
Namely, the observations of each task are individually embedded into a common low-dimensional space, and learning occurs jointly in this space.~\citep{teh2017distral,d2019sharing}.

Despite these empirical successes, theoretical explorations to understand the benefits of such \textit{joint training} in RL have been limited. While 
in the supervised learning literature, the benefit of multi-task training is well-studied~\citep[e.g.,][]{baxter1995learning,baxter2000model,ben2003exploiting,du2020few, tripa, tripb}, obtaining a similar understanding in the setting of RL is more challenging. 
For one, a sufficiently flexible yet useful notion of ``task relatedness'' is difficult to formulate in RL, which involves both rewards and transition dynamics. Secondly, an algorithm using such a relatedness measure must carefully balance exploration and exploitation, while appropriately handling inevitable inaccuracies in the learned representation and how these can compound over the horizon.

Given these existing shortcomings in the literature, in this work we aim to theoretically analyze the benefit of learning joint representations for multi-task RL.
We begin by formalizing the underlying similarity -- i.e., \emph{task relatedness} -- between multiple tasks. 
Leveraging recent results on \emph{linearly factored} or \emph{low-rank} MDPs~\citep{agarwal2020flambe,yang2020reinforcement,nachum2021provable}, we assume that there exists a state-action representation such that all tasks admit linear transition dynamics with respect to this representation. 
In this setting, any one task may exhibit distinctly
different dynamics from the remaining tasks while still maintaining a common and learnable structure.
Under such a shared representation, we quantify the benefit -- in terms of regret -- given by using a sufficiently accurate approximate representation, and we pair this result with an online algorithm for simultaneously learning and using such a representation.
Our results provide a clear understanding of the trade-offs associated with leveraging a jointly  learned representation in the setting of RL as a function of the dimension of the shared representation, the number of tasks, and the dimension of the raw state observations (see our main results in Section~\ref{section::shared_structure}). We show that in the presence of bilinear optimization oracle there exists a computationally efficient algorithm for multitask \textsf{MatrixRL} via a reduction to quadratic programming (see Section~\ref{section::computational_efficient_algorithm}). This is in contrast with previous work that have studied multi task RL in other function approximation models such as~\cite{hu2021near}. We develop a general regret analysis technique to shave off a $\sqrt{H}$ factor from the the regret upper bounds of episodic linear problems and apply it both to the original \textsf{MatrixRL} rates (see Section~\ref{section::preliminaries}) as well as \textsf{Shared-MatrixRL} (see Section~\ref{section::shared_structure}).

\section{Related Work}
As mentioned above, multi-task learning is well-studied in the supervised learning literature. The predominant mechanism for performing multi-task learning in these previous works is analogous to our own, namely, parameterizing a classifier as a composition of two functions, one of which is common to all tasks and another which is unique to each class~\citep{baxter1995learning,maurer2006bounds,du2020few, tripa, tripb}.
As in our own work, these previous works generally rely on an assumption that an optimal hypothesis with the desired compositional form exists, although some work has explored alternative assumptions~\citep{ben2003exploiting}. 

More closely related to our own setting is the work of~\citet{d2019sharing}, which theoretically analyze approximate dynamic programming in the context of multi-task learning. In this setting, an approximate value function is learned, with a common representation used to parameterize this value function. %
However, it is important to note that approximate dynamic programming is distinct from RL, as it ignores the difficult exploration problem associated with learning from one's own collected data. In contrast, our analysis is specifically tailored to an online learning scenario: where one of the main challenges is deriving multi-task learning bounds which carefully balance exploration and exploitation jointly across all tasks.

Works that do consider the online learning setting include~\citet{yang2020impact} and~\citet{hu2021near}. The first of these considers a linear bandit setting and is not immediately applicable to RL; moreover, this work imposes additional structural conditions on the linear features of the bandit problem which effectively require the action features to sufficiently cover all possible directions.
The second work by~\citet{hu2021near} is closer to ours and considers the linear RL setting. Our `sharedness' assumptions and results are a generalization of those studied in~\citet{hu2021near}. Because~\citet{hu2021near} studies a value-based approach, the assumption is that all the underlying linear parameter weights of the task's $Q$-functions lie in (or near) the same subspace. Because we are studying a model based approach we move beyond this sharedness structure to instead study the setting where the task model matrices share a common factorization. This is the reason our bounds have a dependence on $d'$ as well as on $d$ and $r$. Our regret guarantees are similar to those in~\citet{hu2021near}, despite taking a different approach with distinct derivations. We don't see this as a limitation but rather as an indication to the wider research community that there is a potential opportunity to develop a unifying analysis of RL methods in the presence of shared representation learning that could subsume both value-based and model-based methods. We believe this to be an interesting and exciting avenue for future research. Moreover, we show (see Section~\ref{section::computational_efficient_algorithm}) that as long as we have access to an oracle for joint least squares matrix factorization, the optimization problem required to find the policy to execute at time $t$ can be solved efficiently. This is in stark contrast with other approaches where even solving the necessary joint optimization problem over the task family to find the policy to execute at any given time can be an intractable problem. Recently, other works~\cite{cheng2022provable,agarwal2022provable,lu2022provable} have explored more general versions of our shared task assumptions where the learner may have access only to a family of representation functions and is tasked with learning a viable representation while interacting with multiple tasks at once. They show that shared representation learning is advantageous when compared to learning a single representation per task. These are very close in spirit to the study we present here and should be thought of as a successor works to ours. 

In this work we consider a model for task relatedness inspired by~\cite{tripb}, where we assume the underlying model of the MDP dynamics have a shared low rank representation. Other models of the relationship between related tasks are possible. Most notably~\citet{muller2022meta} and~\citet{moskovitz2022towards}: in~\citet{muller2022meta} the authors consider the question of learning an appropriate `bias' vector for regularizing the MatrixRL algorithm. This allows them to show that in case the variance of the models in the family is small, performance (in this case measured in the form of regret) in a test task can be substantially better.  The authors of~\cite{moskovitz2022towards} tackle a similar issue. In their work they show that under the assumption that the optimal policies are similar across tasks in the family, it is possible to learn a useful default policy such that a policy gradient algorithm that regularizes towards it can learn an optimal policy for a target task much more efficiently than an algorithm regularizing towards the uniform policy. We leave the task of generalizing our work to the setting of a set or distribution train tasks for the purpose of solving a test task for future work.

\section{Preliminaries}\label{section::preliminaries}

Formally,
we consider the setting of episodic reinforcement learning proposed in~\citep{yang2020reinforcement} where an an agent explores an MDP $(\mathcal{S}, \mathcal{A}, \mathbb{P}, r,H)$ with state space $\mathcal{S}$, action space $\mathcal{A}$ and known reward function $r: \mathcal{S}\times \mathcal{A} \rightarrow [0,1]$ whose transition dynamics are given by the feature embedding,
\begin{align*}
    \Pr(\tilde{s} | s, a) = \phi(s,a)^\top \Mstar \psi(\tilde{s})
\end{align*}
The learner receives a noiseless  reward $r(s, a)$ which for simplicity we assume is known. All interactions between a policy and the MDP is of length $H$. For any policy $\pi$, state $s$, action $a$ and $h \in [H]$ we define $V_h^{\pi}(s), Q_h^\pi(s,a)$ as the value and $Q$ functions of policy $\pi$. Our objective is to design algorithms with small regret, defined as
\begin{equation*}
    R(NH) = \sum_{n=1}^N V_1^{\pi_\star}(s_{n,1}) - V_1^{\pi_n }(s_{n,1}),
\end{equation*}
Where $\pi_\star$ corresponds to the optimal policy, $\pi_n$ is the algorithm's policy during time-step $n$ and $s_{n,1}$ are the initial states during the $n$th episode. 

The algorithm in~\cite{yang2020reinforcement} works by building an estimator $\widetilde{\M}_n$ of the matrix $\Mstar$ at time $n$ using the data collected so far. We use the notation $t = (n,h)$ (i.e. episode $n \leq N$ and stage $h \leq H$), to denote the state-action-state triplets $(s_t, a_t, \tilde{s}_t)$ where $\tilde{s}_t = s_{t+1}$. For simplicity we denote the associated features by:
\begin{align*}
    \phi_t = \phi(s_t, a_t) \in \mathbb{R}^d \text{, } \psi_t = \psi(\tilde{s}_t) \in \mathbb{R}^{d'}    \text{and }\M_\star \in \mathbb{R}^{d \times d'}.
\end{align*}

Denote $\boldsymbol{\Psi} \in \mathbb{R}^{ | \mathcal{S}| \times d'}$ as the matrix whose rows equal $\psi(s)$ for all $s \in \mathcal{S}$ and let $\K_{\psi} = \sum_{\ts} \psi(\ts) \psi(\ts)^\top$. For any matrix $\mathbf{B}$ we use $\mathbf{B}[:, i]$ to refer to $\mathbf{B}$'s $i-$th column. We use the notation $\| \cdot \|_F$ to denote the Frobenius norm of a matrix and $\| \cdot \|_2,\| \cdot \|_\infty$ the $l_2$ and $l_\infty$ norms of a vector. We will make the following assumptions regarding the norms of $\M_\star$ and the feature maps $\phi$ and $\psi$.

\begin{assumption}[Boundedness]\label{assumption::boundedness}
The feature maps $\phi$ and $\psi$ satisfy  $\| \phi(s, a) \|_2 \leq L_\phi$, $\| \psi(s) \|_2 \leq L_\psi$ and $\|  \M_\star[  :, i]  \|_2 \leq S$ for all $s,a \in \mathcal{S}\times\mathcal{A}$ and $i \in [d']$ and some known values $L_\phi, L_\psi$ and $S$. And therefore $\|\M_\star\|_F \leq \sqrt{ d'} S$ %

\end{assumption}
We also consider the following two assumptions on feature regularity, both present in~\cite{yang2020reinforcement}.

\begin{assumption}\label{assumption::feature_regularity}[Feature Regularity]
For all $\mathbf{v} \in \mathbb{R}^{|\mathcal{S}|}$, $\| \Psi^\top \mathbf{v}\|_\infty \leq C_{\psi} \| \mathbf{v}\|_\infty$, and $\| \Psi \mathbf{K}_\psi^{-1} \|_{2, \infty}\leq C'_{\psi}$, where $\| \mathbf{Y} \|_{2,\infty} = \max_i \sqrt{\sum_j \mathbf{Y}_{i,j}^2 }$ is the $2, \infty$ norm (infinity norm over the $l_2$ norm of $\mathbf{Y}$'s columns). 
\end{assumption}

We will also prove sharper results under a more refined feature regularity assumption,

\begin{assumption}\label{assumption::stronger_feature_regularity}[Stronger Feature Regularity]
For all $\mathbf{v} \in \mathbb{R}^{|\mathcal{S}|}$, $\| \Psi^\top \mathbf{v}\|_2 \leq C_{\psi} \| \mathbf{v}\|_\infty$, and $\| \Psi \mathbf{K}_\psi^{-1} \|_{2, \infty}\leq C'_{\psi}$, where $\| \mathbf{Y} \|_{2,\infty} = \max_i \sqrt{\sum_j \mathbf{Y}_{i,j}^2 }$ is the $2, \infty$ norm (infinity norm over the $l_2$ norm of $\mathbf{Y}$'s columns). 
\end{assumption}
As it is explained in~\cite{yang2020reinforcement}, this assumption can be satisfied when $\Psi$ is a set of sparse features or if $\Psi$ is a set of highly concentrated features. 

The matrix estimator $\widetilde{\M}_n$ considered by~\cite{yang2020reinforcement} equals:
\begin{equation*}
    \widetilde{\M}_n = \left[ \Sigma_n \right]^{-1} \sum_{n' < n, h\leq H} \phi_{n', h} \psi_{n', h}^\top \K_\psi^{-1}.
\end{equation*}
Where $$\Sigma_n = \lambda \mathbb{I} + \sum_{n' < n,h\leq H} \phi_{n', h} \phi_{n', h}^\top$$ and $\Sigma_{n,h} = \Sigma_n + \sum_{h'<h} \phi_{n,h'}\phi_{n,h'}^\top$.
 It is easy to see that $\widetilde{\M}_n$ is the solution to the ridge regression problem:
\begin{align}
    \widetilde{\M}_n = \argmin_{\M} &\sum_{n'<n, h\leq H} \| \psi^\top_{n', h}\K_{\psi}^{-1} - \phi_{n', h}^\top \M \|_2^2 +  \lambda \| \M \|_F^2. \label{eq::least_squares_objective}
\end{align} 
It can be shown that with high probability and for all $t$ simultaneously all $\widetilde{\M}_n$ lie in a vicinity of $\M_\star$. 

\begin{restatable}{lemma}{confidenceboundssimplelemma}\label{lemma::concentration_M}
For all $\delta  \in (0,1)$ with probability at least $1-\delta$ for all $n \in \mathbb{N}$ simultaneously,
\begin{align*}
    \M_\star &\in \{ \M \in \mathbb{R}^{d\times d'} : \| (\Sigma_n)^{1/2} ( \M - \widetilde{\M}_n) \|_{2,1}\leq d'\sqrt{\beta_ n}\} := \mathbf{U}^{1,2}_n.\\
        \M_\star &\in \{ \M \in \mathbb{R}^{d\times d'} : \| (\Sigma_n)^{1/2} ( \M - \widetilde{\M}_n) \|_{F} \leq \sqrt{d' \beta_n}\} := \mathbf{U}_n^{F}.
\end{align*}
Where $\| \mathbf{B}\|_{2,1}$ denotes the $l_1$ norm of the $l_2$ norm of the columns of $\mathbf{B}$ while $\| \mathbf{B}\|_F$ corresponds to the Frobenius norm, $\sqrt{\beta_n} = R\sqrt{ d \log\left( \frac{d'+ d' n H L_\phi^2/\lambda}{\delta}\right) } + \sqrt{\lambda}S$ and $R = \| \mathbf{K}_\psi^{-1} \| L_\psi + S L_\phi  $. 
\end{restatable}
The proof of Lemma~\ref{lemma::concentration_M} can be found in Appendix~\ref{section::preliminaries_proofs}.  We can make use of Lemma~\ref{lemma::concentration_M} to show a regret guarantee for the \textsf{MatrixRL} algorithm from~\cite{yang2020reinforcement} (see Algorithm~\ref{alg:MatrixRL_basic}). Let's revisit the optimistic value function construction of the \textsf{MatrixRL} algorithm,
\begin{align}
\pushleft{    \forall (s,a) \in \mathcal{S} \times \mathcal{A}: \quad Q_{n, H+1}(s,a) = 0 \text{ and } \forall h \in [H]: }\notag\\
Q_{n,h}(s,a) = r(s,a) +  \max_{\M \in  \mathbf{U}_n^{1,2}} \phi(s,a)^\top \mathbf{M} \Psi^\top V_{n,h+1} \label{equation::optimistic_q_function}
\end{align}
where
\begin{equation*}
    V_{n,h}(s) = \Pi_{[0,H]}\left[  \max_a Q_{n,h}(s,a)  \right] \quad \forall s, a, n, h.
\end{equation*}
 $\Pi_{[0,H]}$ denotes the coordinate-wise clipping/projection operator onto the $[0,H]$ interval. 
\begin{algorithm}%
\begin{algorithmic}[1]
\State \textbf{Input: } An episodic MDP environment $\mathcal{M} = ( \mathcal{S}, \mathcal{A}, P, s_0, r, H)$, features $\phi: \mathcal{S} \times \mathcal{A} \rightarrow \mathbb{R}^d$ and $\psi : \mathcal{S} \rightarrow \mathbb{R}^{d'}$, probability parameter $\delta \in (0,1)$. 
\State \textbf{Initialize: } $\Sigma_1 \leftarrow \mathbb{I} \in \mathbb{R}^{d\times d}$, $\M_1 \leftarrow \mathbf{0} \in \mathbb{R}^{d \times d'}$. 
\State \textbf{for} episode $n=1, \cdots, N$:
\State \qquad Solve for  $\widetilde{\M}_n$.
\State  \qquad Let $\{ Q_{n,h}\}$ be given by Equation~\ref{equation::optimistic_q_function} using $\mathbf{U}_n^{1,2}$ and $\beta_n$ as in Lemma~\ref{lemma::concentration_M}. 
\State  \qquad \textbf{For} stage $h=1, \cdots, H$:
    \State \qquad \quad Let the current state be $s_{n,h}$ .
    \State \qquad \quad Play action $a_{n,h} = \argmax_{a \in \mathcal{A}} Q_{n,h}(s_{n,h},a)$ .
    \State \qquad \quad Record the next state $s_{n,h+1}$.

\State $\Sigma_{n+1} \leftarrow \Sigma_n + \sum_{h\leq H} \phi_{n,h} \phi_{n,h}^\top $.
\State Compute $\widetilde{\M}_{n+1}$ using \eqref{eq::least_squares_objective}. 
\end{algorithmic}
\caption{MatrixRL.}
\label{alg:MatrixRL_basic}
\end{algorithm}

Let's define the ``good" event of probability at least $1-\delta$ where Lemma~\ref{lemma::concentration_M} holds as $\mathcal{E}$. We'll be making heavy use of the following `determinant lemma', 

\begin{lemma}[Determinant Lemma]
\label{lemma:det_lemma}
(Lemma C.3 from~\citep{pacchiano2020regret})
For any sequence of vectors $\mathbf{x}_1,\ldots,\mathbf{x}_M \in \mathbb{R}^d$ such that $\lv \mathbf{x}_q\rv_2 \le L$ for all $q \in [N]$. Given a $\lambda \ge 0$ define $\mathbf{D}_1 := \lambda \mathbf{I}$ and for $\ell \in \{2,\ldots,M+1\}$ define $\mathbf{D}_\ell := \lambda \mathbf{I} + \sum_{q=1}^{\ell-1} \mathbf{x}_q \mathbf{x}_q^{\top}$. Then for all $M\in \mathbb{N}$ and $b > 0$
\begin{equation}\label{equation::bounding_log_det_ratio}
    \log\left(\frac{\mathrm{det}(\mathbf{D}_{M+1})}{\mathrm{det}(\lambda \mathbf{I} )} \right) \leq d \log\left( 1+\frac{  M L^2 }{\lambda d}\right) . 
\end{equation}
and
\begin{equation*}
\sum_{q=1}^M \min\left\{b,\| \mathbf{x}_q \|^2_{\mathbf{D}^{-1}_{q}}\right\}  \leq (1+b) d \log\left( 1 + \frac{M L^2}{\lambda d}\right).
\end{equation*}
\end{lemma}

Our first result is to derive a sharper regret guarantee for the \textsf{MatrixRL} algorithm than in~\cite{yang2020reinforcement},
\begin{restatable}{theorem}{theoremmatrixrlrefined}\label{lemma::regret_guarantee_simple_matrixrl}
The regret satisfies,
\begin{align*}
    R(NH) &\leq  8H \sqrt{NH\log\left( \frac{6\log NH}{\delta}\right)  } +  2\sqrt{2\gamma_NNHd\log\left( 1+\frac{NHL_\phi^2}{\lambda d}\right)}    +   \\
    &\quad \quad 2L_\phi Hd\sqrt{\frac{\gamma_N}{\lambda}}  \log\left( 1+\frac{  N L_\phi^2 }{\lambda d}\right)
\end{align*}
\begin{enumerate}
    \item Under Assumption~\ref{assumption::feature_regularity}, %
\begin{small}
\begin{align*}
   \sqrt{\gamma_N} &=   2C_\psi H d' \sqrt{\beta_N} = 2C_\psi H d'\left( R\sqrt{ d \log\left( \frac{d'+ d' N H L_\phi^2/\lambda}{\delta}\right) } + \sqrt{\lambda}S \right)
\end{align*}
\end{small}
\item Under the stronger Assumption~\ref{assumption::stronger_feature_regularity},
\begin{small}
\begin{align*}
   \sqrt{ \gamma_N} &= 2C_\psi H \sqrt{d' \beta_N} = 2C_\psi H \sqrt{d'}\left( R\sqrt{ d \log\left( \frac{d'+ d' N H L_\phi^2/\lambda}{\delta}\right) } + \sqrt{\lambda}S \right) 
\end{align*}
\end{small}
\end{enumerate}
with probability at least $1-2\delta$. 
\end{restatable}

The proof of Lemma~\ref{lemma::regret_guarantee_simple_matrixrl} can be found in Appendix~\ref{section::proof_lemma_simple_matrixrl}. In contrast with the regret guarantees of~\cite{yang2020reinforcement}, our bounds have a dependence on $H^{3/2}$ as opposed to $H^2$. We achieve this by using the following ``lazy" version of the commonly used determinant lemma in the bandits/RL literature.

\begin{restatable}{lemma}{lemmaupperboundinversenormsgeneralized}\label{lemma::every_n_norm_vs_every_n_h_generalized}
Let $\mathbf{x}_{n,h} \in \mathbb{R}^{\tilde d}$ satisfying $\| \mathbf{x}_{n,h}\|\leq L$ for some $\tilde d \in \mathbb{N}$ and let $\mathbf{D}_{n,h} \in \mathbb{R}^{\tilde d \times \tilde d}$ be a family of positive semidefinite matrices for $n \in \mathbb{N}$ and $1 \leq h \leq H$ such that $\lambda \mathbf{I} \preceq \mathbf{D}_{n,h} \preceq\mathbf{D}_{n',h'}$ if $(n,h) \leq (n',h')$ in the lexicographic order (i.e. $n' > n$ or $h' \geq h$ when $n = n'$). Define $\mathbf{D}_{n} = \mathbf{D}_{n-1, H}$ and $\mathbf{D}_1 = \lambda \mathbf{I}$. The following inequalities hold,

\begin{equation}\label{equation::upper_bound_sigma_n_generalized}
    \sum_{n=1}^N \sum_{h=1}^H \| \mathbf{x}_{n,h}\|_{\mathbf{D}_n^{-1}} \leq \sum_{n=1}^N \sum_{h=1}^H 2 \| \mathbf{x}_{n,h}\|_{\mathbf{D}_{n,h}^{-1}} + \frac{2HL}{\sqrt{\lambda}} \log\left(\frac{\mathrm{det}( \mathbf{D}_{N+1} ) }{\mathrm{det}(\lambda \mathbf{I})}\right).
\end{equation}
\end{restatable}

The proof of Lemma~\ref{lemma::every_n_norm_vs_every_n_h_generalized} can be found in Appendix~\ref{section::additional_technical_results}. As a corollary of Lemma~\ref{lemma::every_n_norm_vs_every_n_h_generalized},

\begin{corollary}\label{corollary::transform_A_n_to_A_n_h}
The following inequalities hold,
\begin{equation}\label{equation::upper_bound_sigma_n}
     \sum_{n=1}^N \sum_{h=1}^H \| \phi_{n,h}\|_{\Sigma_n^{-1}} \leq \sum_{n=1}^N \sum_{h=1}^H 2\| \phi_{n,h}\|_{\Sigma_{n,h}^{-1}} +\frac{2L_\phi Hd}{\sqrt{\lambda}} \log\left( 1+\frac{  N H L_\phi^2 }{\lambda d}\right).
\end{equation}
\end{corollary}

\begin{proof}
As an immediate consequence of Lemma~\ref{lemma::every_n_norm_vs_every_n_h_generalized} by setting $\mathbf{x}_{n,h} = \phi_{n,h}$ and $\mathbf{D}_{n,h} = \Sigma_{n,h}$,
\begin{equation*}
    \sum_{n=1}^N \sum_{h=1}^H \| \phi_{n,h}\|_{\Sigma_n^{-1}} \leq \sum_{n=1}^N \sum_{h=1}^H 2 \| \phi_{n,h}\|_{\Sigma_{n,h}^{-1}} + \frac{2HL}{\sqrt{\lambda}} \log\left(\frac{\mathrm{det}( \Sigma_{N+1} ) }{\mathrm{det}(\lambda \mathbf{I})}\right)
\end{equation*}
Equation~\ref{equation::bounding_log_det_ratio} from Lemma~\ref{lemma:det_lemma} implies,
\begin{equation*}
   \log\left(  \frac{\mathrm{det}( \Sigma_{N+1} ) }{\mathrm{det}(\lambda \mathbf{I})} \right) \leq  d \log\left( 1+\frac{  NH L_\phi^2 }{\lambda d}\right) . 
\end{equation*}
The result follows.
\end{proof}

Corollary~\ref{corollary::transform_A_n_to_A_n_h} allows us to transform a sum of inverse $\Sigma_n^{-1}$ to a sum of inverse $\Sigma_{n,h }^{-1}$ norms. This transformation comes at the cost of a $2$ factor and a logarithmic cost with a $dH$ multiplier. Since it can be shown that $\sum_{n=1}^N \sum_{h=1}^H \| \phi_{n,h}\|_{\Sigma_{n,h}^{-1}} = \widetilde{\mathcal{O}}( \sqrt{dNH} ) $ where $\widetilde{\mathcal{O}}(\cdot)$ hides logarithmic factors, we conclude that $\sum_{n=1}^N \sum_{h=1}^H \| \phi_{n,h}\|_{\Sigma_n^{-1}} = \widetilde{\mathcal{O}}( \sqrt{dNH} )$. This allows us to save a $\sqrt{H}$ factor in our final regret bound. Lemma~\ref{lemma::every_n_norm_vs_every_n_h_generalized} and Corollary~\ref{corollary::transform_A_n_to_A_n_h} can be applied to any episodic linear setting and can be used to shave off a $\sqrt{H}$ factor form other episodic stationary linear models beyond \textsf{MatrixRL}.

\section{Shared Structure Model}\label{section::shared_structure}

 In this work we are concerned with understanding conditions under which sequential learning can be made more sample-efficient when simultaneously training in the presence of several related tasks. In contrast with other works that are concerned with the problem of learning from a set of related source tasks before engaging with a new target task, we are interested in understanding what benefits can be derived simultaneously from joint representation training across multiple RL problems. We borrow the subspace sharedness model from~\cite{yang2020impact} and generalize it from the setting of linear bandits to the previously described \textsf{MatrixRL} setting. We begin by assuming the learner has access to $P$ tasks encoded by the matrices $\{ \Mstar^{(p)} \}^P_{p =1}$ with known reward functions $\{r^{(p)}\}_{p=1}^P$. We make the assumption the transitions factorize as $\Mstar^{(p)} = \mathbf{B}_\star \A_\star^{(p)}$ where $\mathbf{B}_\star \in \mathbb{R}^{d \times r}$ is a projection operator\footnote{Recall that a linear operator $\mathbf{P}$ is a projection if $\mathbf{P}^2 \mathbf{v} = \mathbf{P}\mathbf{v}$. }
 and $\A^{(p)}_\star\in \mathbb{R}^{r \times d'}$. We require all of the matrices $\Mstar^{(p)}$ to satisfy Assumption~\ref{assumption::boundedness}, so $\| \Mstar^{(p)} \|_F = \|\A_\star^{(p)}\|_F \leq \sqrt{d' }S $.

 We are interested in designing an algorithm that bounds the ``shared regret", defined as
\begin{equation*}
    R_P(NH) = \sum_{n=1}^N \sum_{p=1}^P V_1^{\pi_\star^{(p)}}(s_{n,1}^{(p)}) - V_1^{\pi_n^{(p)} }(s^{(p)}_{n,1}),
\end{equation*}
where $s_{n,1}^{(p)}$ is the starting state for task $p$ in epsisode $n$, $\pi_n^{(p)}$ is the policy used by task $p$ during epsiode $n$, and $\pi_\star^{(p)}$ is the optimal policy of task $p$. Notice that instead of optimizing the usual form of the single task regret, here we are interested in minimizing the aggregate regret incurred across all tasks. The learner's objective is to leverage the shared structure among the tasks to incur a regret $R_{P}(NH)$ smaller then what is obtained by learning each task in isolation--a shared regret equal to $P$ times the single-task \textsf{MatrixRL} regret upper bound.

In this framework, the transition dynamics across MDPs are coupled because the agent's feature embedding of state-action pairs lie in a common low-dimensional subspace. If the learner had knowledge of $\mathbf{B}_\star$, they would be able to use projected features of the form $\tilde{\phi}(s,a) = \mathbf{B}_\star \phi(s,a)$ in their exploration. This would allow the learner to incur regret scaling only in $r$, independently of $d$. Although it is impossible to completely eliminate the $d$-dependence without apriori knowledge of $\mathbf{B}_\star$, we show that in some cases it is possible to improve the $d$-dependence. Our main result can be summarized as follows,
\begin{theorem}[Informal]
There exists an algorithm for joint learning over a set of related tasks $\{ \Mstar^{(p)}= \mathbf{B}_\star \mathbf{A}_{\star}^{(p)}\}_{p \in [P]}$ that achieves a regret of
\begin{equation*}
  R_P(NH) = \widetilde{\mathcal{O}}\left( \left( Hd\sqrt{rP} +  HP\sqrt{rd}\right)\sqrt{NH} \right),
\end{equation*}
with high probability, where $\widetilde{\mathcal{O}}$ hides logarithmic factors %
\end{theorem}

Recall that for an isolated task in order to recover an estimator $\widetilde{\M}$ of $\Mstar$ given $n-1$ trajectories of horizon $H$ we solve $d'$ independent ridge regression problems (one per column) as defined by Equation~\ref{eq::least_squares_objective}.

In the multi-task setting with shared structure, we instead consider the following quadratic objective that weaves together the estimation of the task-specific $\{ \mathbf{A}^{p} \}_{p=1}^{P}$ parameters with that of the shared $\mathbf{B}$ projection matrix.\footnote{Our results will also be true when the $\psi, \phi$ maps are task-dependent. In this case, the only change to our results would require making $\mathbf{K}_{\psi}$ task-dependent.}

\begin{align}\label{eq::multitask_least_squares_objective}
\argmin_{\substack{\mathbf{B} \in \mathcal{P}_{d,r}, \\  \| \mathbf{A}^{(1)} \|_F \leq \sqrt{d'} S, \cdots, \|\mathbf{A}^{(P)} \|_F \leq \sqrt{d'}S } } F(\mathbf{B},\mathbf{A}^{(1)}, \cdots, \mathbf{A}^{(P)}  ) \qquad \qquad\qquad\qquad\qquad\qquad\qquad\qquad\qquad\\
\qquad F(\mathbf{B},\mathbf{A}^{(1)}, \cdots, \mathbf{A}^{(P)}  ) = \sum_{p \in [P]} \lambda \| \mathbf{A}^{(p)} \|_F^2 +\sum_{n'<n, h\leq H} \left\| \left(\psi^{(p)}_{n', h}\right)^\top\K_{\psi}^{-1} - \left(\phi^{(p)}_{n', h}\right)^\top \left( \mathbf{B} \mathbf{A}^{(p)} \right) \right\|_2^2 \notag
\end{align}

Where $\mathcal{P}_{d,r}$ corresponds to the set of all $d\times r$ projection matrices with $r$ orthonormal columns and the search space for $\mathbf{A}^{(p)}$ is the Frobenius ball of radius $\sqrt{d'}S$ in the space of matrices $\mathbb{R}^{r \times d'}$.

Notice that by virtue of the orthogonality of $\mathbf{B}$'s columns (i.e. $\mathbf{B}^\top \mathbf{B} = \mathbb{I}_r$ ) the regularizer satisfies $\| \mathbf{B} \mathbf{A}^{(p)}\|_F^2 = \| \mathbf{A}^{(p)}\|_F^2$. We use the notation $\widetilde{\mathbf{B}}_n, \widetilde{\mathbf{A}}_n^{(1)}, \cdots, \widetilde{\mathbf{A}}_n^{(P)}$ to refer to the resulting estimators for the shared projection matrix and the low rank dynamics matrices for each of the tasks $p = 1, \cdots, P$ right before the $n$th batch of $P$ trajectories is collected.

We start by proving a series of data dependent bounds on the estimates $\widetilde{\mathbf{B}}, \widetilde{\mathbf{A}}^{(1)}_n, \cdots, \widetilde{\mathbf{A}}^{(P)}_n$ that will serve as the analogous shared-structure versions of Lemma~\ref{lemma::concentration_M}.

Now we show a bound for the data-dependent distance between $\widetilde{\mathbf{B}}_n, \widetilde{\mathbf{A}}_n^{(1)}, \cdots, \widetilde{\mathbf{A}}_n^{(P)}$ and the true parameters $\mathbf{B}_\star, \mathbf{A}^{(1)}_\star, \cdots, \mathbf{A}_\star^{(P)}$.

\begin{restatable}{lemma}{concentrationsharedlemma}\label{lemma:concentration_shared}
For any $\delta \in (0,1)$ the following bound holds,
\begin{small}
\begin{align*}
     \sum_{p \in [P]} \lambda \left\|  \widetilde{\mathbf{A}}_n^{(p)} \right\|_F^2 +  \frac{1}{2}  \left\| \left(\Sigma^{(p)}_n\right)^{1/2} \left(\mathbf{B}_\star \mathbf{A}_\star^{(p)} - \widetilde{\mathbf{B}}_n \widetilde{\mathbf{A}}_n^{(p)} \right) \right\|_F^2 
     \leq  \beta'_{nH}(\delta) +     \sum_{p \in [P]} \lambda \| \mathbf{A}_\star^{(p)}\|_F^2 
\end{align*}
\end{small}
with probability at least $1-\delta$ for all $n \in \mathbb{N}$ and where 
\begin{align*}
\pushleft{\beta'_{nH}(\delta) = 1 + L_\phi S + \frac{b^2}{2R^2} + } \\
\quad (12R^2 +  b) \Big( 2 \ln \ln \left(2 \left(nHP\right)\right) + 3 + \ln\frac{1}{\delta} + 
(dr + rd'P)\left( \ln(5S) + \ln{nHP} + \ln{2RL_\phi} \right) \Big)
\end{align*}
And $b =2R d' S L_{\psi} $.
\end{restatable}
The proof of Lemma~\ref{lemma:concentration_shared} can be found in Appendix~\ref{section::proof_shared_representation_ridge}. In contrast with the results of Lemma~\ref{lemma::concentration_M}, the guarantees of Lemma~\ref{lemma:concentration_shared} apply to the sum of the errors across all $P$ tasks. As we'll see in the coming discussion this is the main source of difficulties in designing a reinforcement learning algorithm that successfully makes use of this result to construct optimistic value functions. We can use Lemma~\ref{lemma:concentration_shared} to obtain the following high probability confidence interval jointly around $\widetilde{\mathbf{B}}_n$ and  $\{\widetilde{\mathbf{A}}_n^{(p)}\}_{p=1}^P$, which is one of our main results:

\begin{lemma}\label{lemma::concentration_M_shared}
For any $\delta \in (0,1)$ with probability at least $1-\delta$ for all $n \in \mathbb{N}$ simultaneously,
\begin{align*}
   \Big\{ \mathbf{M}_\star^{(p)} =  \mathbf{B}_\star\mathbf{A}_\star^{(p)} \Big\}_{p=1}^P &\in \left\{  \{ \mathbf{B}\mathbf{A}^{(p)}\}_{p=1}^P \text{ s.t. } \sum_{p}  \left\| \left(\Sigma^{(p)}_n\right)^{1/2} \left(\mathbf{B} \mathbf{A}^{(p)} - \widetilde{\mathbf{B}}_n \widetilde{\mathbf{A}}_n^{(p)} \right) \right\|_F^2\leq \gamma_{n}(\delta)    \right \}\\
    &\subseteq \underbrace{\left\{ \{\mathbf{M}^{(p)}\}_{p=1}^P \text{ s.t. } \sum_{p}  \left\| \left(\Sigma^{(p)}_n\right)^{1/2} \left( \mathbf{M}^{(p)} - \widetilde{\mathbf{B}}_n \widetilde{\mathbf{A}}_n^{(p)} \right) \right\|_F^2\leq \gamma_{n}(\delta)    \right \}}_{ := \widetilde{\mathbf{U}}_n^{F}(\delta)}
\end{align*}
where $\gamma_n(\delta)= 2\beta'_{n}(\delta) + 2P \sqrt{d'} S \lambda$ and $\beta'_n$ is defined as in Lemma~\ref{lemma:concentration_shared}.
\end{lemma}
\begin{proof}
Lemma~\ref{lemma:concentration_shared} implies that with probability at least $1-\delta$ for all $n \in \mathbb{N}$,
\begin{small}
\begin{align*}
     \sum_{p \in [P]} \lambda \left\|  \widetilde{\mathbf{A}}_n^{(p)} \right\|_F^2 +  \frac{1}{2}  \left\| \left(\Sigma^{(p)}_n\right)^{1/2} \left(\mathbf{B}_\star \mathbf{A}_\star^{(p)} - \widetilde{\mathbf{B}}_n \widetilde{\mathbf{A}}_n^{(p)} \right) \right\|_F^2 
     \leq  \beta'_{nH}(\delta) +     \sum_{p \in [P]} \lambda \| \mathbf{A}_\star^{(p)}\|_F^2 
\end{align*}
\end{small}
Since $\| \mathbf{A}_\star^{(p)} \|_F \leq \sqrt{d'}S$, this implies that 
\begin{align*}
    \sum_{p \in [P] } \left\| \left(\Sigma^{(p)}_n\right)^{1/2} \left(\mathbf{B}_\star \mathbf{A}_\star^{(p)} - \widetilde{\mathbf{B}}_n \widetilde{\mathbf{A}}_n^{(p)} \right) \right\|_F^2    \leq 2\beta'_{nH}(\delta) + 2P\sqrt{d'}S \lambda.
\end{align*}
The result follows.
\end{proof}

From here on we use the name $\mathcal{E}'$ to define the event of Lemma~\ref{lemma::concentration_M_shared} where the sum of the square of the confidence intervals across all tasks is bounded by $\gamma_n(\delta)$. Lemma~\ref{lemma::concentration_M_shared} implies $\mathbb{P}(\mathcal{E}') \geq1-\delta $.

\begin{algorithm}[h]
  \begin{algorithmic}[1]
\State \textbf{Input: } Episodic MDP environments $\{\mathcal{M}^{(p)}\}_{p\in [P]} = ( \mathcal{S}, \mathcal{A}, \mathbb{P}^{(p)}, s_0, r, H)$, features $\phi^{(p)}: \mathcal{S} \times \mathcal{A} \rightarrow \mathbb{R}^d$ and $\psi^{(p)} : \mathcal{S} \rightarrow \mathbb{R}^{d'}$, probability parameter $\delta \in (0,1)$. 
\State \textbf{Initialize: } $\{\Sigma^{(p)}_1 \leftarrow \mathbb{I} \in \mathbb{R}^{d\times d}\}_{p=1}^P$, $\{\M_1^{(p)} \leftarrow \mathbf{0} \in \mathbb{R}^{d \times d'}\}$. 
\State \textbf{For} episode $n=1, \cdots, N$:
\State \qquad Solve Problem~\ref{eq::multitask_least_squares_objective} and compute $\widetilde{\mathbf{B}}_n, \widetilde{\mathbf{A}}_n^{(1)}, \cdots, \widetilde{\mathbf{A}}_n^{(P)}$.
\State \qquad Let $\{ Q^{(p)}_{n,h}\}_{p=1}^P$ be given by $Q^{(p)}_{n, h}(s,a) = Q_{n,h}^{(p)}(s,a, \{ \widebar{\mathbf{M}}_n^{(p)} \}_{p=1}^P)$
\State \qquad where, 
\begin{equation}\label{equation::argmax_parametric_value_functions}
    \{ \widebar{\mathbf{M}}_n^{(p)}\}_{p=1}^P = \argmax_{\{\mathbf{M}^{(p)}\}_{p=1}^P \in \widetilde{\mathbf{U}}_n^F(\delta) }  \sum_p V_{n,1}^{(p)}(s_{n,1}^{(p)}, \{ \mathbf{M}^{(p)}\}_{p=1}^P).
\end{equation}
\State \qquad Where $\{s_{n,1}^{(p)}\}_{p=1}$ is the set of first states seen at the start of their episodes by all tasks.  
\State \qquad \textbf{For} $p=1, \cdots , P$:
\State\qquad \qquad \textbf{For} stage $h=1, \cdots, H$ :
\State \qquad\qquad \quad     Let the current state be $s^{(p)}_{n,h}$.
\State \qquad\qquad \quad    Play action $a^{(p)}_{n,h} = \argmax_{a \in \mathcal{A}} Q^{(p)}_{n,h}(s^{(p)}_{n,h},a)$.
\State \qquad\qquad \quad    Record the next state $s^{(p)}_{n,h+1}$.
\State \qquad\qquad  Update $\Sigma^{(p)}_{n+1} \leftarrow \Sigma^{(p)}_n + \sum_{h\leq H} \left(\phi^{(p)}_{n,h}\right) \left(\phi^{(p)}_{n,h}\right)^\top $ for all $p \in[P]$.

\caption{\textsf{Shared-MatrixRL.}}
\label{alg:MatrixRL_shared}
\end{algorithmic}
\end{algorithm}

We now introduce the \textsf{Shared-MatrixRL} algorithm. In contrast with the simple \textsf{MatrixRL} in Algorithm~\ref{alg:MatrixRL_basic}, \textsf{Shared-MatrixRL} makes use of a shared confidence interval for the $\widetilde{\mathbf{B}}_n,\{\widetilde{\mathbf{A}}_n^{(p)}\}_{p=1}^P$ matrices. We define the following optimistic $Q-$functions for the task family, 
\begin{align*}
\forall  \{ \mathbf{M}^{(p)} \}_{p \in [P]} \text{ and }   \forall (s, a) \in \mathcal{S} \times \mathcal{A}:
 \quad Q^{(p)}_{n, H+1}(s,a, \{ \mathbf{M}^{(p)} \}_{p=1}^P  ) = 0 \quad \forall p \in [P] \text{and }\forall h \in [H]: \\
    Q^{(p)}_{n,h}(s,a, \{ \mathbf{M}^{(p)} \}_{p \in [P]} ) =  r^{(p)}(s_p, a_p) + \phi^{(p)}(s_p, a_p)^\top  \mathbf{M}^{(p)} \left(\boldsymbol{\Psi}^{(p)}\right)^\top V^{(p)}_{n,h+1}( \{ \mathbf{M}^{(p)} \}_{p=1}^{P} ) 
\end{align*}
 where $V^{(p)}_{n,h+1}( \{ \mathbf{M}^{(p)} \}_{p = 1}^{P} )$ is a vector of dimension $|\mathcal{S}|$ corresponding to the value functions of task $p$ under model $\mathbf{M}^{(p)}$. For all $s, a, n, h$,
\begin{align*}
 V^{(p)}_{n,h}(s, \{ \mathbf{M}^{(p)} \}_{p = 1}^P ) =   \Pi_{[0,H]}\left[  \max_a Q^{(p)}_{n,h}(s,a, \{ \mathbf{M}^{(p)} \}_{p = 1}^P )  \right].   
 \end{align*}

The definition of the parametric $Q$ functions $Q_{n, h}^{(p)}(s,a, \{ \mathbf{M}^{(p)} \}_{p=1}^P )$ and value functions $V^{(p)}_{n,h}(s, \{ \mathbf{M}^{(p)} \}_{p = 1}^P )$ is required to define the joint optimistic objective for the set of $P$ tasks of Equation~\ref{equation::argmax_parametric_value_functions}. We define the optimistic value functions as,
 \begin{align*}
    Q^{(p)}_{n, h}(s,a) &= Q_{n,h}^{(p)}(s,a, \{ \widebar{\mathbf{M}}_n^{(p)} \}_{p=1}^P), \qquad
    V_{n, h}^{(p)}(s) =  V_{n, h}^{(p)}(s, \{ \widebar{\mathbf{M}}_n^{(p)} \}_{p=1}^P)
\end{align*}
The optimization problem of Equation~\ref{equation::argmax_parametric_value_functions} requires to solve for $\{ \widebar{\mathbf{M}}_n^{(p)} \}_{p=1}^P$ optimizes the sum of values as `seen' from the initial states $\{s_{n,1}^{(p)}\}_{p \in [P]}$ of the $P$ tasks at the beginning of the $n$th episode. This form of optimism is required to ensure the constraint $\{\widebar{\mathbf{M}}^{(p)}\}_{p=1}^P \subset \widetilde{\mathbf{U}}_n^F(\delta)$ is satisfied.

\paragraph{Limitations.} \textsf{Shared-MatrixRL} works in a similar way to the single task Matrix RL algorithm; a policy is executed in each of the component tasks based on a series of optimistic $Q$ values. The data collected by the learner is then used to update the component models via Equation~\ref{eq::multitask_least_squares_objective}.  The chief difference in our approach to the multi task setting lies in the definition of the shared $Q$ functions. This is what allows us to make use of the shared confidence interval of Lemma~\ref{lemma::concentration_M_shared}. Unfortunately this means the computation of the `optimistic models' $\{ \widebar{\mathbf{M}}_n^{(p)} \}_{p=1}^P$ is intractable since it requires the computation and storage of the $Q$ values $ Q_{n,h}^{(p)}(s,a,  \{ \mathbf{M}_n^{(p)} \}_{p=1}^P)$ for all feasible values of $\{ \mathbf{M}_n^{(p)}\}_{p \in [P]}$ and then solve for $ \{ \widebar{\mathbf{M}}_n^{(p)} \}_{p =1}^P$. This situation is not as severe as it seems since the computation of the optimistic $Q$ functions in the original \textsf{MatrixRL} algorithm (and even in the OFUL algorithm for linear bandits~\cite{abbasi2011improved}) is also an intractable problem. Another potential drawback of Algorithm~\ref{alg:MatrixRL_shared} is its requirement to have knowledge of the initial states $\{s^{(p)}_{n,1}\}_{p=1}^P$. An astute reader may posit it to be possible to overcome this issue by using Thompson Sampling~\cite{agrawal2013thompson,abeille2017linear}. In this case we would sample a set of models $\{ \widebar{\mathbf{M}}_n^{(p)} \}_{p=1}^P$ from block gaussian distribution where each block is centered around each $\widetilde{\mathbf{M}}_n^{(p)}$. Sampling from this posterior does not require knowledge of $\{s_{n,1}^{(p)}\}_{p=1}^P$. Unfortunately, this strategy would cause the degradation of the regret upper bound to a level that is not competitive with the strategy of solving each task independently. We leave the removal of the assumption on $\{ s_{n,1}^{(p)}\}_{p=1}^P$ as future work.

In order to prove the \textsf{Shared-MatrixRL} satisfies a satisfactory sublinear regret guarantee we start by showing optimism holds for the shared representations parameterized by $\{ \widebar{\mathbf{M}}_n^{(p)}\}_{p = 1}^P$.

\begin{lemma}[Optimism]\label{lemma::shared_optimism} Whenever $\mathcal{E}'$ holds,
\begin{equation*}
    \sum_{p \in [P]} V_1^{\pi_\star^{(p)}}(s_{n,1}^{(p)}) \leq \sum_{p \in [P]} V_{n,1}^{(p)}(s_{n,1}^{(p)}).
\end{equation*}
 
\end{lemma}

\begin{proof}
Since $$V_{n,1}^{(p)}(s_{n,1}^{(p)}) = V_{n,1}\left(s_{n,1}^{(p)}, \{ \widebar{\mathbf{M}}_n^{(p)} \}_{p=1}^P\right) $$the definition of $ \{ \widebar{\mathbf{M}}_n^{(p)} \}_{p=1}^P$ implies that,
\begin{align*}
   \sum_{p \in [P]} V_{n,1}(s_{n,1}^{(p)}, \{ \widebar{\mathbf{M}}_n^{(p)} \}_{p=1}^P)  \geq
   \sum_{p \in [P]} V_{n,1}(s_{n,1}^{(p)},  \{\mathbf{B}_\star \mathbf{A}_\star^{(p)} \}_{p=1}^P) 
\end{align*}
Since $V_{n,1}(s_{n,1}^{(p)}, \mathbf{B}_\star, \{ \mathbf{A}_\star^{(p)} \}_{p=1}^P)  =  V_1^{\pi_\star^{(p)}}(s_{n,1}^{(p)})$, the result follows.
\end{proof}

Similarly we can use our confidence interval bounds to prove the following bound on the bellman error.

\begin{restatable}{lemma}{sumoptimismhelper}\label{lemma::upper_bound_bellman_error}
If Assumption~\ref{assumption::stronger_feature_regularity} holds and $\mathcal{E}'$ is true then for $h \in [H]$, 
\begin{align*}
\sum_{p \in [P]} Q^{(p)}_{n,h}(s^{(p)}_{n,h}, a^{(p)}_{n,h}) - 
    \left( r(s^{(p)}_{n,h}, a^{(p)}_{n,h}) + \mathbb{P}^{(p)}(\cdot | s^{(p)}_{n,h}, a^{(p)}_{n,h})^\top V^{(p)}_{n,h+1}\right) \\
    \leq  2C_\psi H  \sqrt{\gamma_{n}(\delta) \sum_{p \in [P]} \| \phi_{n,h}^{(p)}\|^2_{\left(\Sigma_n^{(p)}\right)^{-1}} } 
\end{align*}

\end{restatable}

The proof of Lemma~\ref{lemma::upper_bound_bellman_error} can be found in Appendix~\ref{section::bound_bellman_error}.  Having established that optimism holds, we can use a similar set of techniques as in the proof of Theorem~\ref{lemma::regret_guarantee_simple_matrixrl} to show a regret guarantee. First we derive Corollary~\ref{corollary::n_inverse_to_n_h_inverse_multitask}, an equivalent version to Corollary~\ref{corollary::transform_A_n_to_A_n_h}. This allows us to maintain the $\sqrt{H}$ factor improvement in the multitask setting. This result is a consequence of Lemma~\ref{lemma::every_n_norm_vs_every_n_h_generalized}. 

\begin{corollary}\label{corollary::n_inverse_to_n_h_inverse_multitask}
The following inequalities hold,
\begin{equation}\label{equation::upper_bound_sigma_n_shared}
     \sum_{n=1}^N \sum_{h=1}^H \sqrt{\sum_{p \in [P]} \| \phi^{(p)}_{n,h}\|^2_{\left(\Sigma^{(p)}_n\right)^{-1}} }\leq \sum_{n=1}^N \sum_{h=1}^H  2\sqrt{\sum_{p \in [P]} \| \phi^{(p)} _{n,h}\|^2_{\left(\Sigma^{(p)}_{n,h}\right)^{-1}}} +\frac{2L_\phi HdP}{\sqrt{\lambda}} \log\left( 1+\frac{  N H L_\phi^2 }{\lambda d}\right).
\end{equation}
\end{corollary}

\begin{proof}
Define $NHP$ variables $\mathbf{x}_{n,h,p} \in \mathbb{R}^{d P}$ ordered lexicographically and satisfying $\mathbf{x}_{n,h} = (\phi_{n,h}^{(1)}, \cdots, \phi_{n,h}^{(P)})$ where $\phi_{n,h}^{(p)}$ is located in the $p-$th $d$ dimensional slot of $\mathbf{x}_{n,h}$ for all $p \in [P]$. In this case, $\mathbf{D}_{n,h}$ is a block diagonal matrix (with $d \times d$ diagonal blocks equal to $\Sigma_{n,h}$) such that $\| \mathbf{x}_{n,h} \|_{\mathbf{D}^{-1}_{n,h}} = \sqrt{\sum_{p \in [P]} \| \phi_{n,h}^{(p)} \|^2_{\left(\Sigma_{n,h}^{(p)}\right)^{-1}}}$. By definition $\| \mathbf{x}_{n,h} \|_{\mathbf{D}^{-1}_n}  = \sqrt{ \sum_{p \in [P]} \| \phi_{n,h}^{(p)} \|^2_{\left(\Sigma_n^{(p)}\right)^{-1}}}$. As a consequence of Lemma~\ref{lemma::every_n_norm_vs_every_n_h_generalized},
\begin{equation*}
     \sum_{n=1}^N \sum_{h=1}^H \sqrt{\sum_{p \in [P]} \| \phi^{(p)}_{n,h}\|^2_{\left(\Sigma^{(p)}_n\right)^{-1}} } \leq \sum_{n=1}^N \sum_{h=1}^H 2\sqrt{\sum_{p \in [P]} \| \phi^{(p)} _{n,h}\|^2_{\left(\Sigma^{(p)}_{n,h}\right)^{-1}}} +\frac{2L_\phi H}{\sqrt{\lambda}} \log\left(\frac{\mathrm{det}( \mathbf{D}_{N+1} ) }{\mathrm{det}(\lambda \mathbf{I}_{dP})}\right).
\end{equation*}

Where we have used the notation $\mathbf{I}_{s}$ to denote the $s\times s$ dimensional identity matrix. By definition of $\mathbf{D}_{N+1}$ we see that $\mathrm{det}( \mathbf{D}_{N+1}) = \prod_{p =1}^P \mathrm{det}( \Sigma^{(p)}_{N+1} )$ and therefore, 

\begin{equation*}
    \log\left(\frac{\mathrm{det}( \mathbf{D}_{N+1} ) }{\mathrm{det}(\lambda \mathbf{I}_{dP})}\right) = \sum_{p=1}^P \log\left(\frac{\mathrm{det}( \Sigma^{(p)}_{N+1} ) }{\mathrm{det}(\lambda \mathbf{I}_{d})}\right) \leq Pd \log\left( 1 + \frac{NHL_\phi^2}{\lambda d}\right).
\end{equation*}
Where the last inequality follows from Equation~\ref{equation::bounding_log_det_ratio} in Lemma~\ref{lemma:det_lemma}. The result follows.
\end{proof}

Similar to Corollary~\ref{corollary::transform_A_n_to_A_n_h}, the result of Corollary~\ref{corollary::n_inverse_to_n_h_inverse_multitask} allows us to transform inverse norms defined by the matrices $\left(\Sigma_n^{(p)}\right)^{-1}$, into inverse norms defined by the matrices   $\left(\Sigma_{n,h}^{(p)}\right)^{-1}$, at a constant multiplicative cost plus a logarithmic term with a $dHP$ multiplier.

\begin{restatable}{theorem}{maintheoremshared}\label{theorem::regret_guarantee_shared}
The regret of \textsf{Shared-MatrixRL} satisfies, %
\begin{align*}
    R_P(NH) &\leq  H \sqrt{NHP\log\left( \frac{6\log NH}{\delta}\right)  } + 4C_\psi H^2 dP \left(1+ \frac{L_\phi^2}{\sqrt{\lambda}}\right) \log\left( 1 + \frac{NHL_\phi^2}{\lambda d}\right) \sqrt{\gamma_N(\delta)}  + \\
    &\quad  2C_\psi H \sqrt{\gamma_N(\delta)NH Pd \left(1+ \frac{L_\phi^2}{\sqrt{\lambda}}\right) \log\left( 1 + \frac{NHL_\phi^2}{\lambda d}\right)  }.
\end{align*}

 With probability at least $1-2\delta$. 
\end{restatable}

The proof can be found in Appendix~\ref{appendix:proofs}. Since $\gamma_N(\delta)\approx dr+rP$ (up to logarithmic factors and ignoring polynomial dependencies on $d'$) Theorem~\ref{theorem::regret_guarantee_shared} implies,
\begin{corollary}
The regret of Algorithm~\ref{alg:MatrixRL_shared} satisfies,
\begin{align*}
    R_P(NH) &\leq \widetilde{\mathcal{O}}\left( H\sqrt{NHP} + H\sqrt{dr + rP} \sqrt{NHPd} \right) = \widetilde{\mathcal{O}}\left( \left(  Hd\sqrt{rP} +  HP\sqrt{rd}\right)\sqrt{NH} \right).
\end{align*}
 With probability at least $1-2\delta$. 
\end{corollary}

This result improves upon the shared regret of order $\widetilde{\mathcal{O}}(HdP\sqrt{NH})$ achieved by using the \textsf{MatrixRL} algorithm to learn each task independently. Interestingly, learning the tasks' shared structure only becomes beneficial when
$r \ll d$ and $r \ll P$. To explain this phenomenon observe that the degrees of freedom (i.e. the number of parameters to learn) in \textsf{Shared-MatrixRL} equals $dr + Pr$. The degrees of freedom of running $P$ independent copies of \textsf{MatrixRL} in contrast equals $dPd'$. For shared representation learning to be more efficient than learning each task alone, we require $dr + Prd' \ll dPd'$. This is why for \textsf{Shared-MatrixRL} learning to be truly beneficial (and attain a smaller regret upper bound than running $P$ tasks independently) we require $dr \ll dPd'$ and $Prd' \ll dPd'$. For example when the number of tasks is small and $P \ll r$, learning the shared matrix $\mathbf{B}_\star$ may require more data than learning the $dPd'$ parameters of estimating the models for all $P$ tasks independently. Although we have not developed a lower bound for the specific \textsf{MatrixRL} setting, the results of~\citet{yang2020impact} provide evidence to posit the regret upper bound for \textsf{Shared-MatrixRL} in Theorem~\ref{theorem::regret_guarantee_shared} is optimal.

\section{Computationally Efficient \textsf{Shared-MatrixRL}}\label{section::computational_efficient_algorithm}

Algorithm~\ref{alg:MatrixRL_shared} has two computationally intensive components. First, solving for $\widetilde{\mathbf{B}}_n, \widetilde{\mathbf{A}}_n^{(1)}, \cdots, \widetilde{\mathbf{A}}_n^{(P)}$ and second, solving for Equation~\ref{equation::argmax_parametric_value_functions}. The first objective may be difficult to solve because it involves solving a bilinear quadratic optimization problem. The second one can prove even more challenging first because it requires a way to `store' the parametric value functions $V_{n,1}^{(p)}(s, \{ \mathbf{M}^{(p)}\}_{p=1}^P) $ (these functions may be highly non-linear), and second because solving for Equation~\ref{equation::argmax_parametric_value_functions} involves optimizing a non-convex objective. 

In this section we show that, given access to a computational oracle for Problem~\ref{eq::multitask_least_squares_objective} and assuming $\mathcal{S}, \mathcal{A}$ are finite, there exists a computationally efficient procedure for solving for the joint optimistic objective of Equation~\ref{equation::argmax_parametric_value_functions} of Algorithm~\ref{alg:MatrixRL_shared}. As it is mentioned in the discussion surrounding Equation~7 of~\cite{yang2020reinforcement}, the confidence bonus of Equation~\ref{equation::optimistic_q_function} can be substituted by
\begin{align*}
    Q_{n,h}(s,a) &= r(s,a)  + \phi(s,a)^\top \widetilde{\M}_n \boldsymbol{\Psi}^\top V_{n,h+1} + 2L_{\Psi}  H \sqrt{\beta_n} \|   \phi(s,a) \|_{\Sigma^{-1}_n }
\end{align*}
This corresponds to explicitly solving for the optimistic model maximizing the $Q$ values at state action pair $(s,a)$ and in-episode time $h$. Let $\tau^{(p)}$ be a set of $P$ confidence radii. In the multi-task setting, let's consider enforcing,
\begin{equation*}
    \widebar{\mathbf{M}}_n^{(p)} \in \left\{ \left\| \left(\Sigma^{(p)}_n\right)^{1/2} \left(\mathbf{M}^{(p)} - \widetilde{\mathbf{B}}_n \widetilde{\mathbf{A}}_n^{(p)} \right) \right\|_F \leq \tau^{(p)} \right \} := \widetilde{\mathbf{U}}_n^{F}(\delta, p, \tau^{(p)})
\end{equation*}
If $\sum_{p =1}^P (\tau^{(p)})^2 \leq \gamma_n(\delta)$, we can allow for the per-state maximization of the optimistic models as in the single task setting (see Equation~\ref{equation::optimistic_q_function}) and obviate solving for problem~\ref{equation::argmax_parametric_value_functions} in Algorithm~\ref{alg:MatrixRL_shared}. If we call $\widebar{\mathbf{M}}_n^{(p)}(s,a)$ the model in $\mathcal{\mathbf{U}}_n^F(\delta, p, \tau^{(p)})$ achieving the argmax in the definition $Q^{(p)}_{n,h}(s,a, \tau^{(p)}) = r(s,a) +  \max_{\M \in  \mathbf{U}_n^F(\delta, p, \tau^{(p)})} \phi(s,a)^\top \mathbf{M} \Psi^\top V^{(p)}_{n,h+1}(\tau^{(p)})$.  This is because restricting the individual confidence radii for model $p$ to be upper bounded by $\tau^{(p)}$ for all state action pairs ensures that,
\begin{align*}
   \sum_{p \in [P]} Q^{(p)}_{n,h}( s^{(p)}_{n,h}, a^{(p)}_{n,h}, \tau^{(p)}) &- \left( r(s^{(p)}_{n,h}, a^{(p)}_{n,h}) + \mathbb{P}^{(p)}(\cdot | s_{n,h}, a_{n,h})^\top V^{(p)}_{n,h+1}(\tau^{(p)})\right)  \\
    &\leq \sum_{p \in [P]} \left\|\left(\phi^{(p)}_{n,h}\right)^\top \left( \widebar{\M}^{(p)}_n - \M_\star^{(p)} \right) \right\|_2 \left\| \left(\Psi^{(p)}\right)^\top V^{(p)}_{n,h+1}( \tau^{(p)})\right\|_2 \\
    &\leq \sum_{p \in [P]} C_\psi \left\| V^{(p)}_{n,h+1}(\tau^{(p)})\right\|_\infty \left\| \left(\phi^{(p)}_{n,h}\right)^\top \left( \widebar{\M}^{(p)}_n - \M^{(p)}_\star \right) \right\|_2
    \end{align*}
If $\{ \tau^{(p)}\}_{p=1}^P$ are defined such that $\sum_{p=1}^P \left( \tau^{(p)}\right)^2 \leq \gamma_n(\delta)$ the same arguments as in the proof of Lemma~\ref{lemma::upper_bound_bellman_error} imply,
\begin{align*}
    \sum_{p \in [P]} Q^{(p)}_{n,h}( s^{(p)}_{n,h}, a^{(p)}_{n,h}, \tau^{(p)}) - \left( r(s^{(p)}_{n,h}, a^{(p)}_{n,h}) + \mathbb{P}^{(p)}(\cdot | s_{n,h}, a_{n,h})^\top V^{(p)}_{n,h+1}(\tau^{(p)})\right)  \\
    \leq  2C_\psi H  \sqrt{\gamma_{n}(\delta) \sum_{p \in [P]} \| \phi_{n,h}^{(p)}\|^2_{\left(\Sigma_n^{(p)}\right)^{-1}} } \label{equation::bellman_decomposition_q_tau_values}
\end{align*}

The $Q$ functions $Q_{n,h}^{(p)}(\cdot, \cdot, \tau^{(p)})$ satisfy,
\begin{align*}
    Q^{(p)}_{n,h}(s,a, \tau^{(p)}) &= 
        r^{(p)}(s,a) +  \phi^{(p)}(s,a)^\top  \widetilde{\mathbf{B}}_n \widetilde{\mathbf{A}}_n^{(p)} \left(\boldsymbol{\Psi}^{(p)} \right)^\top V^{(p)}_{n,h+1} (\tau^{(p)} )+ \\
    &\quad 2L_{\Psi} H \tau^{(p)} \| \phi^{(p)}(s,a) \|_{\left(\Sigma^{(p)}_n\right)^{-1}}
\end{align*}
Where $$V^{(p)}_{n,h+1}(\tau^{(p)}) = \Pi_{[0,H]}\left[  \max_a Q^{(p)}_{n,h}(s,a, \tau^{(p)})  \right] \quad \forall s, a, n, h.$$ 

If $\mathcal{S}, \mathcal{A}$ are finite sets, then for any fixed set of thresholds $\{ \tau^{(p)}\}$, solving for $Q^{(p)}(s,a, \tau^{(p)})$ can be expressed as the solution to a linear program in the variables $V^{(p)}_{n,h+1}$ and $Q_{n,h}^{(p)}$. By adding a quadratic constraint of the form $\sum_{p \in [P]} \left(\tau^{(p)}\right)^2 \leq \gamma_n(\delta)$   the resulting optimization problem over all tasks $p \in [P]$ becomes the convex Quadratically Constrained Linear Program (QCLP),
\begin{align*}
    \max_{p \in [P]} &\quad \sum_{p=1}^P V^{(p)}_{n, 1}(s_{n,1}^{(p)}, \tau^{(p)})\text{ s.t.} \sum_{p \in [P]} \left(\tau^{(p)}\right)^2 \leq \gamma_n(\delta),
\end{align*}
and thus it will take $\mathrm{poly}\left(\frac{1}{NH}\right)$ operations to arrive at an $\frac{1}{N^2H^2}$ approximate solution for this problem. This is enough to guarantee optimism up to an overall error of order $\frac{1}{NH}$. See the discussion in Chapter 4 of~\cite{boyd2004convex} on how to solve QCLP problems efficiently. %

\section{Conclusion}
In this work we are the first to analyze the problem of joint training across a set of related Markov Decision Processes. We show that when the training tasks' transition dynamics can be embedded in a  common low-dimensional subspace of dimension $r$, a joint training algorithm can obtain regret $\widetilde{\mathcal{O}}\left( \left(  Hd\sqrt{rP} +  HP\sqrt{rd}\right)\sqrt{NH} \right)$ as opposed to $\widetilde{\mathcal{O}}(HdP\sqrt{NH})$ -- the regret of learning each task separately ignoring the shared task structure. Our training method solves a quadratic optimization problem that jointly penalizes the shared and  task-dependent model parameters (see Equation~\ref{eq::multitask_least_squares_objective}). We expect the techniques we have introduced in this work, including the multitask least squares objective of Equation~\ref{eq::multitask_least_squares_objective} and the parametric $Q$ functions $Q^{(p)}_{n,h}(s,a, \{ \mathbf{M}^{(p)} \}_{p \in [P]} )$, to have applications in other MDP models with function approximation--such as Linear MDPs~\cite{jun2019bilinear,zhou2021nearly} amongst others.

\begin{ack}
This was conducted as a result of the Google BAIR Commons program at UC Berkeley.
\end{ack}

\bibliographystyle{abbrvnat}
\bibliography{ref}

\begin{thebibliography}{34}
\providecommand{\natexlab}[1]{#1}
\providecommand{\url}[1]{\texttt{#1}}
\expandafter\ifx\csname urlstyle\endcsname\relax
  \providecommand{\doi}[1]{doi: #1}\else
  \providecommand{\doi}{doi: \begingroup \urlstyle{rm}\Url}\fi

\bibitem[Abbasi-Yadkori et~al.(2011)Abbasi-Yadkori, P{\'a}l, and
  Szepesv{\'a}ri]{abbasi2011improved}
Y.~Abbasi-Yadkori, D.~P{\'a}l, and C.~Szepesv{\'a}ri.
\newblock Improved algorithms for linear stochastic bandits.
\newblock \emph{Advances in neural information processing systems},
  24:\penalty0 2312--2320, 2011.

\bibitem[Abeille and Lazaric(2017)]{abeille2017linear}
M.~Abeille and A.~Lazaric.
\newblock Linear thompson sampling revisited.
\newblock In \emph{Artificial Intelligence and Statistics}, pages 176--184.
  PMLR, 2017.

\bibitem[Agarwal et~al.(2020)Agarwal, Kakade, Krishnamurthy, and
  Sun]{agarwal2020flambe}
A.~Agarwal, S.~Kakade, A.~Krishnamurthy, and W.~Sun.
\newblock Flambe: Structural complexity and representation learning of low rank
  mdps.
\newblock \emph{arXiv preprint arXiv:2006.10814}, 2020.

\bibitem[Agarwal et~al.(2022)Agarwal, Song, Sun, Wang, Wang, and
  Zhang]{agarwal2022provable}
A.~Agarwal, Y.~Song, W.~Sun, K.~Wang, M.~Wang, and X.~Zhang.
\newblock Provable benefits of representational transfer in reinforcement
  learning.
\newblock \emph{arXiv preprint arXiv:2205.14571}, 2022.

\bibitem[Agrawal and Goyal(2013)]{agrawal2013thompson}
S.~Agrawal and N.~Goyal.
\newblock Thompson sampling for contextual bandits with linear payoffs.
\newblock In \emph{International conference on machine learning}, pages
  127--135. PMLR, 2013.

\bibitem[Baxter(1995)]{baxter1995learning}
J.~Baxter.
\newblock Learning internal representations.
\newblock In \emph{Proceedings of the eighth annual conference on Computational
  learning theory}, pages 311--320, 1995.

\bibitem[Baxter(2000)]{baxter2000model}
J.~Baxter.
\newblock A model of inductive bias learning.
\newblock \emph{Journal of artificial intelligence research}, 12:\penalty0
  149--198, 2000.

\bibitem[Ben-David and Schuller(2003)]{ben2003exploiting}
S.~Ben-David and R.~Schuller.
\newblock Exploiting task relatedness for multiple task learning.
\newblock In \emph{Learning theory and kernel machines}, pages 567--580.
  Springer, 2003.

\bibitem[Bertsekas(2019)]{bertsekas2019reinforcement}
D.~P. Bertsekas.
\newblock \emph{Reinforcement learning and optimal control}.
\newblock Athena Scientific Belmont, MA, 2019.

\bibitem[Boyd et~al.(2004)Boyd, Boyd, and Vandenberghe]{boyd2004convex}
S.~Boyd, S.~P. Boyd, and L.~Vandenberghe.
\newblock \emph{Convex optimization}.
\newblock Cambridge university press, 2004.

\bibitem[Cheng et~al.(2022)Cheng, Feng, Yang, Zhang, and
  Liang]{cheng2022provable}
Y.~Cheng, S.~Feng, J.~Yang, H.~Zhang, and Y.~Liang.
\newblock Provable benefit of multitask representation learning in
  reinforcement learning.
\newblock \emph{arXiv preprint arXiv:2206.05900}, 2022.

\bibitem[D'Eramo et~al.(2019)D'Eramo, Tateo, Bonarini, Restelli, and
  Peters]{d2019sharing}
C.~D'Eramo, D.~Tateo, A.~Bonarini, M.~Restelli, and J.~Peters.
\newblock Sharing knowledge in multi-task deep reinforcement learning.
\newblock In \emph{International Conference on Learning Representations}, 2019.

\bibitem[Du et~al.(2020)Du, Hu, Kakade, Lee, and Lei]{du2020few}
S.~S. Du, W.~Hu, S.~M. Kakade, J.~D. Lee, and Q.~Lei.
\newblock Few-shot learning via learning the representation, provably.
\newblock \emph{arXiv preprint arXiv:2002.09434}, 2020.

\bibitem[Du et~al.(2021)Du, Kakade, Lee, Lovett, Mahajan, Sun, and
  Wang]{du2021bilinear}
S.~S. Du, S.~M. Kakade, J.~D. Lee, S.~Lovett, G.~Mahajan, W.~Sun, and R.~Wang.
\newblock Bilinear classes: A structural framework for provable generalization
  in rl.
\newblock \emph{arXiv preprint arXiv:2103.10897}, 2021.

\bibitem[Howard et~al.(2020)Howard, Ramdas, McAuliffe, Sekhon,
  et~al.]{howard2020time}
S.~R. Howard, A.~Ramdas, J.~McAuliffe, J.~Sekhon, et~al.
\newblock Time-uniform chernoff bounds via nonnegative supermartingales.
\newblock \emph{Probability Surveys}, 17:\penalty0 257--317, 2020.

\bibitem[Howard et~al.(2021)Howard, Ramdas, McAuliffe, and
  Sekhon]{howard2018uniform}
S.~R. Howard, A.~Ramdas, J.~McAuliffe, and J.~Sekhon.
\newblock Time-{U}niform, {N}onparametric, {N}onasymptotic {C}onfidence
  {S}equences.
\newblock \emph{The Annals of Statistics}, 49\penalty0 (2):\penalty0
  1055--1080, 2021.

\bibitem[Hu et~al.(2021)Hu, Chen, Jin, Li, and Wang]{hu2021near}
J.~Hu, X.~Chen, C.~Jin, L.~Li, and L.~Wang.
\newblock Near-optimal representation learning for linear bandits and linear
  rl.
\newblock In \emph{International Conference on Machine Learning}, pages
  4349--4358. PMLR, 2021.

\bibitem[Jun et~al.(2019)Jun, Willett, Wright, and Nowak]{jun2019bilinear}
K.-S. Jun, R.~Willett, S.~Wright, and R.~Nowak.
\newblock Bilinear bandits with low-rank structure.
\newblock \emph{arXiv preprint arXiv:1901.02470}, 2019.

\bibitem[Kalashnikov et~al.(2021)Kalashnikov, Varley, Chebotar, Swanson,
  Jonschkowski, Finn, Levine, and Hausman]{kalashnikov2021mt}
D.~Kalashnikov, J.~Varley, Y.~Chebotar, B.~Swanson, R.~Jonschkowski, C.~Finn,
  S.~Levine, and K.~Hausman.
\newblock Mt-opt: Continuous multi-task robotic reinforcement learning at
  scale.
\newblock \emph{arXiv preprint arXiv:2104.08212}, 2021.

\bibitem[Lu et~al.(2022)Lu, Zhao, Du, and Huang]{lu2022provable}
R.~Lu, A.~Zhao, S.~S. Du, and G.~Huang.
\newblock Provable general function class representation learning in multitask
  bandits and mdps.
\newblock \emph{arXiv preprint arXiv:2205.15701}, 2022.

\bibitem[Maurer(2006)]{maurer2006bounds}
A.~Maurer.
\newblock Bounds for linear multi-task learning.
\newblock \emph{The Journal of Machine Learning Research}, 7:\penalty0
  117--139, 2006.

\bibitem[Moskovitz et~al.(2022)Moskovitz, Arbel, Parker-Holder, and
  Pacchiano]{moskovitz2022towards}
T.~Moskovitz, M.~Arbel, J.~Parker-Holder, and A.~Pacchiano.
\newblock Towards an understanding of default policies in multitask policy
  optimization.
\newblock In \emph{International Conference on Artificial Intelligence and
  Statistics}, pages 10661--10686. PMLR, 2022.

\bibitem[M{\"u}ller and Pacchiano(2022)]{muller2022meta}
R.~M{\"u}ller and A.~Pacchiano.
\newblock Meta learning mdps with linear transition models.
\newblock In \emph{International Conference on Artificial Intelligence and
  Statistics}, pages 5928--5948. PMLR, 2022.

\bibitem[Nachum and Yang(2021)]{nachum2021provable}
O.~Nachum and M.~Yang.
\newblock Provable representation learning for imitation with contrastive
  fourier features.
\newblock \emph{arXiv preprint arXiv:2105.12272}, 2021.

\bibitem[Pacchiano et~al.(2020)Pacchiano, Dann, Gentile, and
  Bartlett]{pacchiano2020regret}
A.~Pacchiano, C.~Dann, C.~Gentile, and P.~Bartlett.
\newblock Regret bound balancing and elimination for model selection in bandits
  and rl.
\newblock \emph{arXiv preprint arXiv:2012.13045}, 2020.

\bibitem[Puterman(1990)]{puterman1990markov}
M.~L. Puterman.
\newblock Markov decision processes.
\newblock \emph{Handbooks in operations research and management science},
  2:\penalty0 331--434, 1990.

\bibitem[Sutton(1992)]{sutton1992introduction}
R.~S. Sutton.
\newblock Introduction: The challenge of reinforcement learning.
\newblock In \emph{Reinforcement Learning}, pages 1--3. Springer, 1992.

\bibitem[Teh et~al.(2017)Teh, Bapst, Czarnecki, Quan, Kirkpatrick, Hadsell,
  Heess, and Pascanu]{teh2017distral}
Y.~W. Teh, V.~Bapst, W.~M. Czarnecki, J.~Quan, J.~Kirkpatrick, R.~Hadsell,
  N.~Heess, and R.~Pascanu.
\newblock Distral: Robust multitask reinforcement learning.
\newblock \emph{arXiv preprint arXiv:1707.04175}, 2017.

\bibitem[Tripuraneni et~al.(2020)Tripuraneni, Jordan, and Jin]{tripa}
N.~Tripuraneni, M.~Jordan, and C.~Jin.
\newblock On the theory of transfer learning: The importance of task diversity.
\newblock In H.~Larochelle, M.~Ranzato, R.~Hadsell, M.~F. Balcan, and H.~Lin,
  editors, \emph{Advances in Neural Information Processing Systems}, volume~33,
  pages 7852--7862. Curran Associates, Inc., 2020.
\newblock URL
  \url{https://proceedings.neurips.cc/paper/2020/file/59587bffec1c7846f3e34230141556ae-Paper.pdf}.

\bibitem[Tripuraneni et~al.(2021)Tripuraneni, Jin, and Jordan]{tripb}
N.~Tripuraneni, C.~Jin, and M.~I. Jordan.
\newblock Provable meta-learning of linear representations.
\newblock In M.~Meila and T.~Zhang, editors, \emph{Proceedings of the 38th
  International Conference on Machine Learning, {ICML} 2021, 18-24 July 2021,
  Virtual Event}, volume 139 of \emph{Proceedings of Machine Learning
  Research}, pages 10434--10443. {PMLR}, 2021.
\newblock URL \url{http://proceedings.mlr.press/v139/tripuraneni21a.html}.

\bibitem[Yang et~al.(2020)Yang, Hu, Lee, and Du]{yang2020impact}
J.~Yang, W.~Hu, J.~D. Lee, and S.~S. Du.
\newblock Impact of representation learning in linear bandits.
\newblock In \emph{International Conference on Learning Representations}, 2020.

\bibitem[Yang and Wang(2020)]{yang2020reinforcement}
L.~Yang and M.~Wang.
\newblock Reinforcement learning in feature space: Matrix bandit, kernels, and
  regret bound.
\newblock In \emph{International Conference on Machine Learning}, pages
  10746--10756. PMLR, 2020.

\bibitem[Yu et~al.(2020)Yu, Quillen, He, Julian, Hausman, Finn, and
  Levine]{yu2020meta}
T.~Yu, D.~Quillen, Z.~He, R.~Julian, K.~Hausman, C.~Finn, and S.~Levine.
\newblock Meta-world: A benchmark and evaluation for multi-task and meta
  reinforcement learning.
\newblock In \emph{Conference on Robot Learning}, pages 1094--1100. PMLR, 2020.

\bibitem[Zhou et~al.(2021)Zhou, Gu, and Szepesvari]{zhou2021nearly}
D.~Zhou, Q.~Gu, and C.~Szepesvari.
\newblock Nearly minimax optimal reinforcement learning for linear mixture
  markov decision processes.
\newblock In \emph{Conference on Learning Theory}, pages 4532--4576. PMLR,
  2021.

\end{thebibliography}

\appendix

\section{Proofs from Section~\ref{section::preliminaries}}\label{section::preliminaries_proofs}

\subsection{Proof of Lemma~\ref{lemma::every_n_norm_vs_every_n_h_generalized}}\label{section::additional_technical_results}

\lemmaupperboundinversenormsgeneralized*

\begin{proof}
Define $e_{n,h} = \mathbf{1}\left(  \| \mathbf{x}_{n,h} \|_{\mathbf{D}_{n}^{-1}}   \leq 2\|  \mathbf{x}_{n,h}\|_{\mathbf{D}^{-1}_{n, h}} \right) $. We define $e_{n,h}^c =1-e_{n,h} = \mathbf{1}\left(  \|  \mathbf{x}_{n,h} \|_{\mathbf{D}_{n}^{-1}}   \geq 2\|  \mathbf{x}_{n,h}\|_{\mathbf{D}^{-1}_{n, h}} \right)$. Thus,

\begin{align*}
    \sum_{n=1}^N \sum_{h=1}^H \| \mathbf{x}_{n,h}\|_{\mathbf{D}_n^{-1}} &=\sum_{n=1}^N \sum_{h=1}^H e_{n,h} \| \mathbf{x}_{n,h}\|_{\mathbf{D}_n^{-1}} +  \sum_{n=1}^N \sum_{h=1}^H e^c_{n,h} \| \mathbf{x}_{n,h}\|_{\mathbf{D}_n^{-1}} \\
    &\leq \sum_{n=1}^N \sum_{h=1}^H 2\| \mathbf{x}_{n,h}\|_{\mathbf{D}_{n,h}^{-1}} +   \sum_{n=1}^N \sum_{h=1}^H e^c_{n,h} \frac{L}{\sqrt{\lambda}}.
\end{align*}
Where the inequality holds because $e_{n,h} \| \mathbf{x}_{n,h}\|_{\mathbf{D}_{n}^{-1}} \leq 2\| \mathbf{x}_{n,h}\|_{\mathbf{D}_{n,h}^{-1}}$ and because for all $n,h$ the bound $\| \mathbf{x}_{n,h} \|_{\mathbf{D}_{n}^{-1}} \leq \frac{L}{\sqrt{\lambda}}$ holds. If $e_{n,h}^c =1$ (and therefore $\| \mathbf{x}_{n,h} \|_{\mathbf{D}_{n}^{-1}}   \geq 2\| \mathbf{x}_{n,h}\|_{\mathbf{D}^{-1}_{n, h}}$), Lemma~\ref{lemma::supporting_lin_alg_result}, (setting $\mathbf{B} = \mathbf{D}_{n}^{-1}$ and $\mathbf{C} = \mathbf{D}_{n,h}^{-1}$. These satisfy $\mathbf{B} \succeq \mathbf{C} \succ \mathbf{0}$) implies

\begin{equation*}
4\leq    \frac{\| \mathbf{x}_{n,h} \|^2_{\mathbf{D}^{-1}_{n}}}{\| \mathbf{x}_{n,h} \|^2_{\mathbf{D}^{-1}_{n,h}}} \leq \sup_{\mathbf{x} \neq \mathbf{0} } \frac{\mathbf{x}^\top \left(\mathbf{D}_{n}^{-1}\right) \mathbf{x}}{\mathbf{x}^\top \left(\mathbf{D}_{n,h}^{-1}\right) \mathbf{x}} \leq \frac{\mathrm{det}( \mathbf{D}^{-1}_{n}) }{\mathrm{det}( \mathbf{D}^{-1}_{n,h})} = \frac{\mathrm{det}( \mathbf{D}_{n,h}) }{\mathrm{det}( \mathbf{D}_n)}.
\end{equation*}

Define $E_n = \mathbf{1}\left( \sum_{h} e^c_{n,h} \geq 1\right)$. Notice that $HE_n \geq \sum_{h} e^c_{n,h}$ for all $n$. For all $n$ denote by $h_n = \min \{ h \text{ s.t } e^c_{n,h}=1\}$. In case $\{ h \text{ s.t. } e_{n,h}^c = 1\} = \emptyset$ we define $h_n = H+1$ so that $\mathbf{D}_{n,h_n} = \mathbf{D}_{n+1}$. 

The following telescoping relation holds. 

\begin{equation*}
 \prod_{n=1}^{N} \underbrace{\frac{\mathrm{det}(\mathbf{D}_{n,h_n}) }{\mathrm{det}(\mathbf{D}_{n})}}_{\geq 4^{E_n}} \cdot \underbrace{ \frac{\mathrm{det}(\mathbf{D}_{n+1}) }{\mathrm{det}(\mathbf{D}_{n, h_n})}}_{\geq 1} = \frac{\mathrm{det}(\mathbf{D}_{N+1})}{\mathrm{det}(\mathbf{D}_1)} = \frac{\mathrm{det}(\mathbf{D}_{N+1})}{\mathrm{det}(\lambda \mathbf{I})} 
\end{equation*}
Therefore it must be the case that,
\begin{equation*}
    4^{\sum_{n} E_n} \leq  \frac{\mathrm{det}( \mathbf{D}_{N+1} ) }{\mathrm{det}(\lambda \mathbf{I})}.
\end{equation*}
Thus, $\sum_{n=1}^N E_n \leq 2\log\left(\frac{\mathrm{det}( \mathbf{D}_{N+1} ) }{\mathrm{det}(\lambda \mathbf{I}}\right)$. Finally, this implies, 
$$\sum_{n=1}^N \sum_{h=1}^H e_{n,h}^c \leq H \sum_{n=1}^N E_n \leq 2H \log\left(\frac{\mathrm{det}( \mathbf{D}_{N+1} ) }{\mathrm{det}(\lambda \mathbf{I})}\right).$$
The result follows.
\end{proof}

Recall that we denote the 'good' event that all confidence intervals $\{ \mathbf{U}_n^{1,2} \}_n$ hold at all times $n \in \mathbb{N}$ as $\mathcal{E}$.

\subsection{Proof of Lemma~\ref{lemma::concentration_M}}

\confidenceboundssimplelemma*

\begin{proof}
By definition, $\mathbb{E}\left[ \mathbf{K}_\psi^{-1}  \psi_{n',h}   | s_{n',h}, a_{n',h} \right] =  \left(\M^\star\right)^\top \phi_{n',h} \in \mathbb{R}^{d'}$. Let $D_{n',h}(i) = \psi_{n',h}^\top\mathbf{K}_\psi^{-1}[:, i] -  \phi_{n',h}^\top\left(\M^\star\right)^\top [:,i]$. Notice that,
\begin{equation*}
    | D_{n',h}(i) | \leq  \| \mathbf{K}_\psi^{-1} \| L_\psi + S L_\phi  := R.  \quad \forall n' \in [N], h \in [H].
\end{equation*}
Observe that $\mathbb{E}[ D_{n',h}(i) | F_{n',h}] = 0$ for all $i \in [d],n' \in [N],h \in [H]$. Since $R < \infty$, this implies all the $D_{n',h}(i)$ random variables are conditionally subgaussian with parameter $R$. Problem~\ref{eq::least_squares_objective} can be decomposed into $d'$ independent Ridge Regression problems for each column of $\M_\star$,
\begin{equation*}
   \widetilde{\M}_n[:, i] = \argmin_{\M} \sum_{n'<n, h\leq H} \left\| \psi^\top_{n', h}\K_{\psi}^{-1}[:,i] - \phi_{n', h}^\top \M[:, i] \right\|_2^2 + \lambda \left\| \M[:,i] \right\|_2^2.
\end{equation*}

As a consequence of Theorem~\ref{thm:conf_ellipse_linucb} we see that with probability at least $1-\delta$,

\begin{equation*}
    \underbrace{\| \widetilde{\M}_n[:, i] - \M_\star[:, i] \|_{\Sigma_n} }_{= \| (\Sigma_n)^{1/2}\left( \widetilde{\M}_n[:, i] - \M_\star[:, i]\right) \|_2}\leq R\sqrt{ d \log\left( \frac{d'+ d' n H L_\phi^2/\lambda}{\delta}\right) } + \sqrt{\lambda}S := \sqrt{\beta_{n}}.
\end{equation*}

For all $i \in [d']$ and all $n \in [N]$. Therefore, 

\begin{align*}
\left    \| (\Sigma_n)^{1/2} \left( \widetilde{\M}_n- \M_\star  \right)\right\|_{2, 1} &\leq d' \sqrt{\beta_n}, \qquad \qquad \qquad    \left\| (\Sigma_n)^{1/2} \left( \widetilde{\M}_n- \M_\star  \right)\right\|_{F} &\leq \sqrt{d' \beta_n}
\end{align*}
\end{proof}

\subsection{Regret Guarantees for MatrixRL}

\begin{lemma}[Optimism]\label{lemma::optimism} Suppose $\mathcal{E}$ holds. Then for all $h \in [H]$ and $(s,a) \in \mathcal{S} \times \mathcal{A}$, we have
\begin{equation*}
    Q^\star_h(s,a) \leq Q_{n,h}(s,a)
\end{equation*}

\end{lemma}

\begin{proof}
The same argument as in Lemma 4 of~\cite{yang2020reinforcement} can be used to show this result. 
\end{proof}

Next we show that the confidence balls $\mathbf{U}_n^{1,2}$ and $\mathbf{U}_n^{F}$ can be used to give a strong upper bound for the estimation error. 
\begin{lemma}

For any $\M \in \mathbf{U}_n^{1,2}$ we have,
\begin{align*}
    \| \phi_{s,a}^\top \left(   \M - \widetilde{\M}_n       \right) \|_1 &\leq d' \sqrt{\beta_n} \| \phi_{s,a} \|_{\Sigma_n^{-1}} 
    \end{align*}
For any $\M \in \mathbf{U}_n^F$,
\begin{align*}
    \| \phi_{s,a}^\top \left(   \M - \widetilde{\M}_n       \right) \|_2 &\leq \sqrt{d' \beta_n} \| \phi_{s,a}\|_{\Sigma^{-1}_n}
\end{align*}
\end{lemma}

\begin{proof}
The following inequalities hold,
\begin{align*}
    \| \phi_{s,a}^\top \left(   \M - \widetilde{\M}_n       \right) \|_1 &= \| \phi_{s,a}^\top \left(  \Sigma_n\right)^{-1/2} \left( \Sigma_n \right)^{1/2} \left(   \M - \widetilde{\M}_n       \right)       \|_1 \\
    &\leq \| \phi_{s,a}^\top \left(  \Sigma_n\right)^{-1/2} \|_2 \| \left( \Sigma_n \right)^{1/2} \left(   \M - \widetilde{\M}_n       \right)       \|_{2,1} \\
    &\leq d' \sqrt{\beta_n} \| \phi_{s,a} \|_{\Sigma_n^{-1}} 
\end{align*}
Similarly,
\begin{align*}
    \| \phi_{s,a}^\top \left(   \M - \widetilde{\M}_n       \right) \|_2 &= \| \phi_{s,a}^\top \left(  \Sigma_n\right)^{-1/2} \left( \Sigma_n \right)^{1/2} \left(   \M - \widetilde{\M}_n       \right)       \|_2 \\
    &\leq \| \phi_{s,a}^\top \left(  \Sigma_n\right)^{-1/2} \|_2 \| \left( \Sigma_n \right)^{1/2} \left(   \M - \widetilde{\M}_n       \right)       \|_{F} \\
    &\leq \sqrt{d' \beta_n} \| \phi_{s,a}\|_{\Sigma^{-1}_n}
\end{align*}
\end{proof}

The following lemma holds, 
\begin{lemma}\label{lemma::recursive_bounding_Q}
Suppose that $\mathcal{E}$ holds. Then for $h \in [H]$, we have,
\begin{enumerate}
    \item Under Assumption~\ref{assumption::feature_regularity},
\begin{equation*}
    Q_{n,h}(s_{n,h}, a_{n,h}) - \left( r(s_{n,h}, a_{n,h}) + P(\cdot | s_{n,h}, a_{n,h})^\top V_{n,h+1}\right) \leq  2C_\psi H d' \sqrt{\beta_n} \| \phi_{n,h}\|_{\Sigma_n^{-1}}.
\end{equation*}

\item Under the stronger Assumption~\ref{assumption::stronger_feature_regularity},
\begin{equation*}
    Q_{n,h}(s_{n,h}, a_{n,h}) - \left( r(s_{n,h}, a_{n,h}) + P(\cdot | s_{n,h}, a_{n,h})^\top V_{n,h+1}\right) \leq  2C_\psi H \sqrt{d' \beta_n} \| \phi_{n,h}\|_{\Sigma_n^{-1}}. 
\end{equation*}

\end{enumerate}

\end{lemma}
\begin{proof}
Let $\widetilde{\M} = \argmax_{ \M \in \mathbf{U}_{n}^{1,2}} \phi_{n,h}^\top \M \Psi^\top V_{n,h+1}$. We start by proving the result under Assumption~\ref{assumption::feature_regularity},
\begin{align*}
    Q_{n,h}( s_{n,h}, a_{n,h}) &- \left( r(s_{n,h}, a_{n,h}) + P(\cdot | s_{n,h}, a_{n,h})^\top V_{n,h+1}\right)  \\
    &= \phi_{n,h}^\top \left( \widetilde{\M} - \M_\star\right)\Psi^\top V_{n,h+1}\\
    &\leq \| \phi_{n,h}^\top \left( \widetilde{\M} - \M_\star \right) \|_1 \| \Psi^\top V_{n,h+1}\|_\infty \\
    &\stackrel{(i)}{\leq} C_\psi \| V_{n,h+1}\|_\infty \| \phi_{n,h}^\top \left( \widetilde{\M} - \M_\star \right) \|_1 \\
    &\stackrel{(ii)}{\leq} C_\psi H \| \phi_{n,h}^\top \left( \widetilde{\M} - \M_\star \right) \|_1 \\
    &\leq C_\psi H \left(\| \phi_{n,h}^\top \left( \widetilde{\M} - \widetilde{\M}_n \right) \|_1 + \| \phi_{n,h}^\top \left( \widetilde{\M}_n - \M_\star \right) \|_1  \right) \\
    &\stackrel{(ii)}{\leq} 2C_\psi H d' \sqrt{\beta_n} \| \phi_{n,h}\|_{\Sigma_n^{-1}}
\end{align*}
Inequality $(i)$ follows by Assumption~\ref{assumption::feature_regularity}, and inequality $(ii)$ holds because the range of $V_{n,h+1}$ is bounded by $H$. Finally $(iii)$ is a consequence of conditioning on $\mathcal{E}$ and suing the definition of $\mathbf{U}_n^{1,2}$. 

If instead Assumption~\ref{assumption::stronger_feature_regularity} holds, 

\begin{align*}
    Q_{n,h}( s_{n,h}, a_{n,h}) &- \left( r(s_{n,h}, a_{n,h}) + P(\cdot | s_{n,h}, a_{n,h})^\top V_{n,h+1}\right)  \\
    &= \phi_{n,h}^\top \left( \widetilde{\M} - \M_\star\right)\Psi^\top V_{n,h+1}\\
    &\leq \| \phi_{n,h}^\top \left( \widetilde{\M} - \M_\star \right) \|_2 \| \Psi^\top V_{n,h+1}\|_2 \\
    &\stackrel{(i)}{\leq} C_\psi \| V_{n,h+1}\|_\infty \| \phi_{n,h}^\top \left( \widetilde{\M} - \M_\star \right) \|_2 \\
    &\leq C_\psi H \| \phi_{n,h}^\top \left( \widetilde{\M} - \M_\star \right) \|_2 \\
    &\leq C_\psi H \left(\| \phi_{n,h}^\top \left( \widetilde{\M} - \widetilde{\M}_n \right) \|_2 + \| \phi_{n,h}^\top \left( \widetilde{\M}_n - \M_\star \right) \|_2  \right) \\
    &\stackrel{(ii)}{\leq} 2C_\psi H \sqrt{d' \beta_n} \| \phi_{n,h}\|_{\Sigma_n^{-1}}
\end{align*}

Inequality $(i)$ follows by Assumption~\ref{assumption::stronger_feature_regularity}, and inequality $(ii)$

\end{proof}

\subsection{Proof of Theorem~\ref{lemma::regret_guarantee_simple_matrixrl}}\label{section::proof_lemma_simple_matrixrl}

\theoremmatrixrlrefined*

\begin{proof}
Let's start by conditioning on $\mathcal{E}$. Let $\sqrt{\gamma_{n}} = 2C_\psi H d' \beta_n$ if Assumption~\ref{assumption::feature_regularity} holds and $\sqrt{\gamma_n }= 2C_\psi H \sqrt{d'} \beta_n$ if the stronger Assumption~\ref{assumption::stronger_feature_regularity} holds instead. 

Recall that $R(NH) = \sum_{n=1}^N  V_1^\star( s_{n,1} ) - V_1^{\pi_n}(s_{n,1})$. The optimism property of Lemma~\ref{lemma::optimism} implies that,
\begin{align*}
    R(NH) &\stackrel{(i)}{\leq} \sum_{n=1}^N V_{n,1}(s_{n,1}) - V^{\pi_n}_1(s_{n,1}) \\
    &\stackrel{(ii)}{\leq}\sum_{n=1}^N \sqrt{\gamma_n} \| \phi_{n,1}\|_{\Sigma_n^{-1}} + P(\cdot | s_{n,h}, a_{n,h})^\top \left( V_{n,2}  - V^{\pi_n}_2 \right) \\
    &= \sum_{n=1}^N \sqrt{\gamma_n} \| \phi_{n,1}\|_{\Sigma_n^{-1}} + \underbrace{P(\cdot | s_{n,h}, a_{n,h})^\top \left( V_{n,2}  - V^{\pi_n}_2 \right) - \left(  V_{n,2}(s_{n,2}) - V_{2}^{\pi_n}(s_{n,2}) \right)}_{\delta_{n,2}} +  \\
    &\qquad \quad V_{n,2}(s_{n,2}) - V_{2}^{\pi_n}(s_{n,2})
\end{align*}
Inequality $(i)$ holds by Optimism. Inequality $(ii)$ holds by Lemma~\ref{lemma::recursive_bounding_Q}. A recursive application of this decomposition yields,
\begin{equation*}
    R(NH) \leq \sum_{n=1}^N \sum_{h=1}^H \sqrt{\gamma_n} \| \phi_{n,h}\|_{\Sigma_n^{-1}}  + \delta_{n, h} \leq \sqrt{\gamma_N} \times \underbrace{\sum_{n=1}^N \sum_{h=1}^H  \| \phi_{n,h}\|_{\Sigma_n^{-1}} }_{\mathbf{I}} + \delta_{n, h} 
\end{equation*}

Where the right inequality holds because $\gamma_n$ is increasing (in $n$). The sum  $\sum_{n=1}^N \sum_{h=1}^H \delta_{n,h}$ can easily be bounded by invoking Lemma~\ref{lemma::matingale_concentration_anytime} by observing that $| \delta_{n,h} | \leq 4H$ for all $n,h$. Thus, with probability at least $1-\delta$ for all $n$,

\begin{equation*}
    \sum_{n=1}^N \sum_{h=1}^H \delta_{n,h} \leq 8H \sqrt{NH\log\left( \frac{6\log NH}{\delta}\right)  }.
 \end{equation*}

We proceed to bound $\mathbf{I}$.  By Corollary~\ref{corollary::transform_A_n_to_A_n_h},  equation~\ref{equation::upper_bound_sigma_n}, 
\begin{equation*}
\mathbf{I} \leq \sum_{n=1}^N \sum_{h=1}^H 2\| \phi_{n,h}\|_{\Sigma_{n,h}^{-1}} +\frac{2L_\phi Hd}{\sqrt{\lambda}} \log\left( 1+\frac{  N H L_\phi^2 }{\lambda d}\right).
\end{equation*}

By Lemma~\ref{lemma:det_lemma}, 

\begin{equation*}
   \sum_{n=1}^N \sum_{h=1}^H 2\| \phi_{n,h}\|_{\Sigma_{n,h}^{-1}} \leq  2\sqrt{2NHd\log\left( 1+\frac{NHL_\phi^2}{\lambda d}\right)}.
\end{equation*}

We then conclude that with probability at least $1-\delta$, whenever $\mathcal{E}$ holds,

\begin{align*}
    R(NH) &\leq  8H \sqrt{NH\log\left( \frac{6\log NH}{\delta}\right)  } +  2\sqrt{2\gamma_NNHd\log\left( 1+\frac{NHL_\phi^2}{\lambda d}\right)}    +   \\
    &\quad \quad 2L_\phi Hd\sqrt{\frac{\gamma_N}{\lambda}}  \log\left( 1+\frac{  N L_\phi^2 }{\lambda d}\right)
\end{align*}

Since $\mathcal{E}$ holds with probability at least $1-\delta$ the result follows by a simple union bound.

\end{proof}

\section{Proofs of Section~\ref{section::shared_structure}}\label{appendix:proofs}

Recall that for an isolated task in order to recover an estimator $\widetilde{\M}$ of $\Mstar$ given $n-1$ trajectories of horizon $H$ we run $d'$ independent Ridge Regression problems (one per column) as defined by Equation~\ref{eq::least_squares_objective}.
\begin{equation*}
   \widetilde{\M}_n[:, i] = \argmin_{\M} \widetilde{\mathcal{L}}_n( \M ).
\end{equation*}
Where 
\begin{align*}
\widetilde{\mathcal{L}}_n( \M ) &=   \sum_{n'<n, h\leq H} \left\| \psi^\top_{n', h}\K_{\psi}^{-1}[:,i] - \phi_{n', h}^\top \M[:, i] \right\|_2^2  + \lambda \left\| \M[:,i] \right\|_2^2    
\end{align*}

In order to avoid notational clutter let's call $\widetilde{\psi}_{n', h}^{(p)} = \mathbf{K}_{\psi}^{-1}\psi_{n', h}^{(p)} $. Under this notation, problem~\ref{eq::multitask_least_squares_objective} can be rewritten as,

\begin{align}\label{equation::simplified_multitask_least_squares_objective_appendix}. 
\argmin_{\mathbf{B} \in \mathcal{P}_{d,r}, \mathbf{A}^{(1)}, \cdots, \mathbf{A}^{(P)} \in \mathbb{R}^{r \times d'} }  F(\mathbf{B}, \mathbf{A}^{(1)}, \cdots, \mathbf{A}^{(P)}) \qquad \qquad \qquad\qquad \qquad \qquad \qquad & \\
\qquad  F(\mathbf{B}, \mathbf{A}^{(1)}, \cdots, \mathbf{A}^{(P)}) = \sum_{p \in [P]}  \lambda \| \mathbf{A}^{(p)} \|_F^2 +
 \sum_{n'<n, h\leq H} \| \widetilde{\psi}^{(p)}_{n', h} - \left(\mathbf{B} \mathbf{A}^{(p)} \right)^\top \phi^{(p)}_{n', h} \|_2^2& \notag.
\end{align}

We will make use of the following standard bound on the covering number of the $l_2$ ball.

\begin{lemma}\label{lemma::covering_number_unit_ball}
For any $\epsilon \in(0,1]$ the $\epsilon-$covering number of the Euclidean ball in $\mathbb{R}^d$ with radius $r > 0$ i.e.. $\{ \mathbf{x} \in \mathbb{R}^d : \| \mathbf{x} \|_2 \leq r \} $ is upper bounded by $\left(\frac{1+2r}{\epsilon}\right)^d$.
\end{lemma}

(See for example Lemma D.1 in~\cite{du2021bilinear} )

We will also make use of the following bound on the covering number of the space of $(\mathbf{B}, \mathbf{A}^{(1)}, \cdots, \mathbf{A}^{(p)}\})$ matrices under the norm,
\begin{small}
\begin{align*}
    \| (\mathbf{B}, \mathbf{A}^{(1)}, \cdots, \mathbf{A}^{(p)}\}) - (\bar{\mathbf{B}}, \bar{\mathbf{A}}^{(1)}, \cdots, \bar{\mathbf{A}}^{(p)}\}) \| = 
    \max\left( \| \mathbf{B} - \bar{\mathbf{B}}\|, \| \mathbf{A}^{(1)} - \bar{\mathbf{A}}^{(1)}\| , \cdots,  \| \mathbf{A}^{(p)} - \bar{\mathbf{A}}^{(p)}\|\right) 
\end{align*}
\end{small}

\begin{lemma}\label{lemma::cover_space_matrices_w_projection}
For any $\epsilon \in (0,1)$ the $\epsilon-$covering number under the norm computed by taking the $l_\infty$ norm over $l_2$ norms of the set of matrices $(\mathbf{B}, \mathbf{A}^{(1)}, \cdots, \mathbf{A}^{(p)})$ such that $\mathbf{B} \in \mathbb{R}^{d\times r}$ is a projection matrix with $r$ orthonormal columns and $\mathbf{A}^{(p)} \in \mathbb{R}^{r \times d'}$ are such that $\M^{(p)} = \mathbf{B} \mathbf{A}^{(p)}$ and $\| \M[:, i]\|_2 \leq S$ for all $i \in [d']$ is upper bounded by
\begin{equation*}
    \left( \frac{3}{\epsilon}\right)^{d \times r} \left( \frac{1 + 2S}{\epsilon}  \right)^{r\times d' \times P}
\end{equation*}
\end{lemma}

\begin{proof}

Observe that $\mathbf{B} \in \mathbb{R}^{d \times r}$. The $r$ columns of $\mathbf{B}$ form an orthonormal set. Observe that $\mathbf{B} \in \left\{ \mathbf{X} \in \mathbb{R}^{d \times r } \text{ s.t. }    \| \mathbf{X}[:, i] \|_2 \leq 1 \quad \forall i \in [r]      \right\}$. The results of Lemma~\ref{lemma::covering_number_unit_ball} imply that the covering number of $\left\{ \mathbf{X} \in \mathbb{R}^{d \times r } \text{ s.t. }    \| \mathbf{X}[:, i] \|_2 \leq 1 \quad \forall i \in [r]      \right\} $ is upper bounded by $\left( \frac{3}{\epsilon}\right)^{d \times r}$. 

Recall we are working under the assumption that for all columns $i \in [d']$ of $\M^{(p)} = \mathbf{B}\mathbf{A}^{(p)}$, we have $\| \M^{(p)}[:, i]\|_2\leq S$ for some (known) $S > 0$.

This implies the rows of $\mathbf{A}^{(p)}$ satisfy $\|\mathbf{A}^{(p)}[i, :]\|_2 \leq S$. This in turn implies an upper bound for the $\epsilon-$covering number of the set $\left(\mathbf{A}^{(1)}, \cdots, \mathbf{A}^{(P)} \right)$ of $\left( \frac{1 + 2S}{\epsilon}  \right)^{r\times d' \times P}$.

Putting these together yields an upper bound of $\left( \frac{3}{\epsilon}\right)^{d \times r} \left( \frac{1 + 2S}{\epsilon}  \right)^{r\times d' \times P}$ for the $\epsilon-$covering number of the space of $(\mathbf{B}, \mathbf{A}^{(1)}, \cdots, \mathbf{A}^{(p)}\})$ matrices.

\end{proof}

\subsection{Proof of Lemma~\ref{lemma:concentration_shared}}\label{section::proof_shared_representation_ridge}
\concentrationsharedlemma*

\begin{proof}
For readability, recall that

\begin{align*}
\beta'_{nH}(\delta) &= 1 + L_\phi S + \frac{b^2}{2R^2} + 
(12R^2 +  b) \Big( 2 \ln \ln \left(2 \left(nHP\right)\right) + 3 + \ln\frac{1}{\delta} +\\
&\quad (dr + rd'P)\left( \ln(5S) + \ln{nHP} + \ln{2RL_\phi} \right) \Big)
\end{align*}

By definition $\widetilde{\mathbf{B}}_n, \widetilde{\mathbf{A}}_n^{(1)}, \cdots, \widetilde{\mathbf{A}}_n^{(P)}$ satisfy:

\begin{align*}
    \sum_{p \in [P]} \lambda \| \widetilde{\mathbf{A}}_n^{(p)} \|_F^2 + \sum_{n'<n, h\leq H} \| \widetilde{\psi}^{(p)}_{n', h} - \left(\widetilde{\mathbf{B}}_n \widetilde{\mathbf{A}}_n^{(p)} \right)^\top \phi^{(p)}_{n', h} \|_2^2 \qquad \qquad \qquad \qquad\qquad\qquad\qquad \\
    \qquad \leq \sum_{p \in [P]} \lambda \| \mathbf{A}_\star^{(p)} \|_F^2 + \sum_{n'<n, h\leq H} \| \widetilde{\psi}^{(p)}_{n', h} - \left(\mathbf{B}_\star \mathbf{A}_\star^{(p)} \right)^\top \phi^{(p)}_{n', h} \|_2^2 .
\end{align*}

Let's write $\widetilde{\psi}_{n', h}^{(p)} = \left(\mathbf{B}_\star \mathbf{A}_\star^{(p)} \right)^\top\phi^{(p)}_{n',h} + D^{(p)}_{n', h}$. The random variable $D^{(p)}_{n', h}(i) =  \left(\psi^{(p)}_{n',h}\right)^\top\mathbf{K}_\psi^{-1}[:, i] - \left( \phi^{(p)}_{n',h}\right)^\top\left(\mathbf{B}_\star \mathbf{A}_\star^{(p)}\right)^\top [:,i] \in \mathbb{R}^{d'}$ is conditionally zero mean. Substituting this definition in the formula above and rearranging the resulting terms,

\begin{align}
     \sum_{p \in [P]} \lambda \|  \widetilde{\mathbf{A}}_n^{(p)} \|_F^2 +  \sum_{n'<n, h\leq H}  \| \left(\mathbf{B}_\star \mathbf{A}_\star^{(p)} - \widetilde{\mathbf{B}}_n \widetilde{\mathbf{A}}_n^{(p)} \right)^\top \phi_{n', h}^{(p)}\|^2 \leq \notag \qquad \qquad \qquad \qquad \qquad \\
     \sum_{p \in [P]} \lambda \| \mathbf{A}_\star^{(p)}\|_F^2 + \underbrace{\sum_{n' < n, h \leq H}2 \left\langle D_{n',h}^{(p)} , \left( \widetilde{\mathbf{B}}_n\widetilde{\mathbf{A}}_n^{(p)} - \mathbf{B}_\star \mathbf{A}_\star^{(p)} \right)^\top \phi_{n', h}^{(p)} \right\rangle}_{\mathrm{I}(p)} \label{equation::upper_bound_regularized}
\end{align}

In order to bound the last expression we make use of a covering argument. %

Let $p \in [P]$ and let's focus on obtaining a high probability bound of term $\mathrm{I}(p)$. Let $\epsilon \in (0,1)$ be a number to be defined later. Let's pick a \emph{fixed} element $(\mathbf{B}, \mathbf{A}^{(1)}, \cdots, \mathbf{A}^{(P)})$.

Let's re-index time and instead use tuples $n,h$ with $n \in \mathbb{N}$ and $h \in [H]$ with a lexicographic ordering and use the natural task ordering provided by the task indexes and let's define the martingale difference sequence $\{ Z^{(p)}_{n,h} \}_{n,h\leq H, p\in [P]}$  where $Z^{(p)}_{n,h} = \langle D_{n,h}^{(p)}, \left( \mathbf{B} \mathbf{A}^{(p)} - \mathbf{B}_\star \mathbf{A}_\star^{(p)} \right)^\top \phi_{n,h}^{(p)} \rangle$ for all $p \in [P]$. Recall that $D^{(p)}_{n', h} \in \mathbb{R}^{d'}$ and taht
\begin{equation*}
    | D^{(p)}_{n',h}(i) | \leq  \| \mathbf{K}_\psi^{-1} \| L_\psi + S L_\phi  := R.  \quad \forall n' \in [N], h \in [H], i\in [d'].
\end{equation*}
Observe that $| Z_{n,h}^{(p)}| \leq 2R d' S L_{\phi} = b$ for all $p \in [P]$. %

It is easy to see that $\mathbb{E}\left[ Z^{(p)}_{n,h} | \mathcal{F}_{n,h-1}\right] = 0$ whenever $h > 1$ and $\mathbb{E}\left[ Z^{(p)}_{n,1} | \mathcal{F}_{n-1,H}\right] = 0$ when $h=1$ (and for all $p \in [P]$). For simplicity we use the notation $\{(n', h') \leq (n,h) \}$ to denote all the integer pairs $n \in \mathbb{N}$ and $h \in [H]$ that are less than $(n,h)$ in the lexicographic order. We use the notation $(n', h'-1)$ to denote the preceding point in the lexicographic order to the pair $(n', h')$. This will also hold true for the boundary points so that $(n', 1-1) = (n'-1, H)$.

A simple application of the Cauchy-Schwartz inequality implies the second moments of $\{ Z_{n,h}^{(p)}\}_{n,h \leq h}$ satisfy the following bound for all $p \in [P]$

\begin{align}
 \mathbb{E}\left[   \left(Z^{(p)}_{n,h}\right)^2 | \mathcal{F}_{n, h-1}    \right]  \leq R^2 \left\| \left(\mathbf{B} \mathbf{A}^{(p)} - \mathbf{B}_\star \mathbf{A}_\star^{(p)} \right)^\top \phi_{n', h}^{(p)}\right\|^2  \label{equation::upper_bound_variance1}
\end{align}

Where $R = \| \mathbf{K}_\psi^{-1} \| L_\psi + S L_\phi$ is the probability one upper bound on the magnitude of $D$ defined in Lemma~\ref{lemma::concentration_M}. We will apply the Empirical Bernstein bound of Lemma~\ref{lem:uniform_emp_bernstein} to the martingale sequence $\{ Z^{(p)}_{n,h} \}_{n,h\leq H, p\in [P]}$. 

Let $W_{n,h} = \sum_{(n', h') \leq (n,h) } \sum_{p \in [P] } \mathrm{Var}_{n'-1, h'-1}(  Z^{(p)}_{n',h'}   )  $ be the variance process. As a consequence of Equation~\ref{equation::upper_bound_variance1} we conclude that

\begin{align}
    W_{n,h} &\stackrel{(i)}{\leq} \sum_{(n', h') \leq (n,h) } \sum_{p \in [P]} \mathbb{E}\left[  \left( Z^{(p)}_{n',h'}\right)^2 | \mathcal{F}_{n', h'-1}    \right]  \notag \\
    &\leq \sum_{(n', h') \leq (n,h) } \sum_{p \in [P] } R^2 \left\| \left(\mathbf{B} \mathbf{A}^{(p)} - \mathbf{B}_\star \mathbf{A}_\star^{(p)} \right)^\top \phi_{n', h}^{(p)}\right\|^2 \label{equation::upper_bound_variance}
\end{align}

Where inequality $(i)$ holds because the variance is upper bounded by the second moment. We are ready to use the bound of Lemma~\ref{lem:uniform_emp_bernstein} (with $c= b$). By equating $S_{n,h} = \sum_{(n', h') \leq (n,h)} \sum_{p \in [P]} Z^{(p)}_{n',h'}$ in the definition of the problem and using the upper bound on $W_{n,h}$ from Equation~\ref{equation::upper_bound_variance} yields that for all $\delta \in (0,1)$ with probability at least $1-\delta$ for all $(n,h) \in \mathbb{N} \times [H]$ and all $p \in [P]$. 

\begin{align}
    S_{n,h} &\leq  1.44 \sqrt{\max(W_{n,h} , m) \left( 1.4 \ln \ln \left(2 \left(\max\left(\frac{W_{n,h}}{m} , 1 \right)\right)\right) + \ln \frac{5.2}{\delta}\right)} \notag\\
   & \qquad + 0.41 c  \left( 1.4 \ln \ln \left(2 \left(\max\left(\frac{W_{n,h}}{m} , 1\right)\right)\right) + \ln \frac{5.2}{\delta}\right) \notag \\
   &\stackrel{(i)}{\leq} \frac{1}{4R^2} \max(W_{n,h}, m) + (2.5*4R^2 + 0.41 b) \left( 1.4 \ln \ln \left(2 \left(\max\left(\frac{W_{n,h}}{m} , 1\right)\right)\right) + \ln \frac{5.2}{\delta}\right) \notag \\
   &\stackrel{(ii)}{\leq} \frac{1}{4R^2} \max(W_{n,h}, b^2)  + (2.5 *4R^2 + 0.41 b) \left( 1.4 \ln \ln \left(2 \left(nHP\right)\right) + \ln \frac{5.2}{\delta}\right) \notag \\
   &\leq  \frac{W_{n,h}}{4R^2}  + \frac{b^2}{2R^2} + (2.5*4R^2 + 0.41 b) \left( 1.4 \ln \ln \left(2 \left(nHP\right)\right) + \ln \frac{5.2}{\delta}\right) \label{equation::bounding_S_W}
\end{align}

Where inequality $(i)$ holds because for all $\alpha,\beta \in \mathbb{R}$, $\alpha \beta \leq \alpha^2 + \beta^2$, and $c = b$. Inequality $(ii)$ holds by setting $m = b^2$ (recall $b = 2R d' S L_{\psi}$). Observe now that by Eequation~\ref{equation::upper_bound_variance} we have that 

\begin{equation*}
    \frac{W_{n,h}}{4R^2} \leq \frac{1}{4} \sum_{(n', h') \leq (n,h) } \sum_{p \in [P]}  \left\| \left(\mathbf{B} \mathbf{A}^{(p)} - \mathbf{B}_\star \mathbf{A}_\star^{(p)} \right)^\top \phi_{n', h}^{(p)}\right\|^2.
\end{equation*}

Let $\epsilon' > 0$ and define $\epsilon = \frac{\epsilon'}{2RL_\phi}$.

Notice that for any $(\mathbf{B}, \mathbf{A}^{(1)}, \cdots, \mathbf{A}^{(P)})$ there exists an element of the cover $( \bar{\mathbf{B}}, \bar{\mathbf{A}}^{(1)}, \cdots, \bar{\mathbf{A}}^{(P)})$ such that

 $$\left\| \left(\mathbf{B} \mathbf{A}^{(p)} - \bar{\mathbf{B}}\bar{\mathbf{A}}^{(p)} \right)^\top \phi_{n', h}^{(p)}\right\| \leq \epsilon' \qquad \forall p \in [P]$$ 

 And therefore 

\begin{equation}
\left| \left\| \left(\mathbf{B} \mathbf{A}^{(p)} -\mathbf{B}_\star \mathbf{A}_\star^{(p)} \right)^\top \phi_{n', h}^{(p)}\right\| -  \left\| \left(\bar{\mathbf{B}} \bar{\mathbf{A}}^{(p)} -\mathbf{B}_\star \mathbf{A}_\star^{(p)} \right)^\top \phi_{n', h}^{(p)}\right\| \right| \leq \epsilon' \qquad \forall p \in [P] \label{equation::triangle_inequality_covering}
\end{equation}

Let $\epsilon' = \frac{1}{nHP}$, $\delta \in (0,1)$ and $\delta = \frac{\delta}{\left( \frac{3}{\epsilon}\right)^{d \times r} \left( \frac{1 + 2S}{\epsilon}  \right)^{r\times d' \times P}}$. Denote as $\widebar{\mathbf{B}}_n, \widebar{\mathbf{A}}_n^{(1)}, \cdots, \widebar{\mathbf{A}}_n^{(P)}$ as the point in the covering that is closest to the \emph{random} point $\widetilde{\mathbf{B}}_n, \widetilde{\mathbf{A}}_n^{(1)}, \cdots, \widetilde{\mathbf{A}}_n^{(P)}$. We can conclude that with probability at least $1-\delta$ 

\begin{align*}
    \sum_{p \in [P]} \mathrm{I}(p) &= \sum_{n' < n, h \leq H}  \sum_{p\in [P]} 2 \left\langle D_{n',h}^{(p)} , \left( \widetilde{\mathbf{B}}_n\widetilde{\mathbf{A}}_n^{(p)} - \mathbf{B}_\star \mathbf{A}_\star^{(p)} \right)^\top \phi_{n', h}^{(p)} \right\rangle \\
     &\stackrel{(i)}{\leq} \sum_{n' < n, h \leq H}\sum_{p\in[P]} 2 \left\langle D_{n',h}^{(p)} , \left( \widebar{\mathbf{B}}_n\widebar{\mathbf{A}}_n^{(p)} - \mathbf{B}_\star \mathbf{A}_\star^{(p)} \right)^\top \phi_{n', h}^{(p)} \right\rangle + 2nHRL_\phi P \epsilon \\
     &\stackrel{(ii)}{\leq}\frac{1}{2}  \sum_{(n', h') \leq (n,H) }  \sum_{p \in [P]} \left\| \left(\widebar{\mathbf{B}}_n \widebar{\mathbf{A}}^{(p)}_n - \mathbf{B}_\star \mathbf{A}_\star^{(p)} \right)^\top \phi_{n', h}^{(p)}\right\|^2 + 2nHRL_\phi P \epsilon + \\
     &\qquad \frac{b^2}{2R^2} + (2.5*4R^2 + 0.41 b) \left( 1.4 \ln \ln \left(2 \left(nHP\right)\right) + \ln \frac{5.2}{\delta}\right)\\
     &\stackrel{(iii)}{\leq}\frac{1}{2}  \sum_{(n', h') \leq (n,H) } \sum_{p\in[P]}  \left\| \left(\widetilde{\mathbf{B}}_n \widetilde{\mathbf{A}}_n^{(p)} - \mathbf{B}_\star \mathbf{A}_\star^{(p)} \right)^\top \phi_{n', h}^{(p)}\right\|^2 + 2nHRL_\phi P \epsilon + nH L_\phi S P\epsilon' + \\
      & \qquad \frac{b^2}{2R^2} + (2.5*4R^2 + 0.41 b) \left( 1.4 \ln \ln \left(2 \left(nHP\right)\right) + \ln \frac{5.2}{\delta}\right) \\
     &= \frac{1}{2}  \sum_{(n', h') \leq (n,H) }  \sum_{p\in[P]} \left\| \left(\widetilde{\mathbf{B}}_n \widetilde{\mathbf{A}}_n^{(p)} - \mathbf{B}_\star \mathbf{A}_\star^{(p)} \right)^\top \phi_{n', h}^{(p)}\right\|^2 + 1 + L_\phi S + \\
     &\qquad  \frac{b^2}{2R^2} + (2.5*4R^2 + 0.41 b) \left( 1.4 \ln \ln \left(2 \left(nHP\right)\right) + \ln \frac{5.2}{\delta}\right).
\end{align*}

Inequality $(i)$ holds because $\widebar{\mathbf{B}}_n$ and $\widetilde{\mathbf{B}}_n$ are $\epsilon$ close in the norm resulting of computing the $l_\infty$ norm over the $l_2$ norms of collections of matrices of the form  $(\mathbf{B}, \mathbf{A}_1, \cdots, \mathbf{A}_P) $ and holds with probability $1$. Inequality $(ii)$ holds as a consequence of Equation~\ref{equation::bounding_S_W} and an application of the union bound over the $\epsilon-$cover. Inequality $(iii)$ is a consequence of Equation~\ref{equation::triangle_inequality_covering}. Where the last equality holds because $\epsilon' = \frac{1}{nHP}$ and $\epsilon = \frac{\epsilon'}{2RL_\phi}$ (recall $b = 2R d' S L_{\psi}$). 

Notice that $\delta  = \frac{\delta}{\left( \frac{3}{\epsilon}\right)^{d \times r} \left( \frac{1 + 2S}{\epsilon}  \right)^{r\times d' \times P}  } \leq \frac{\delta}{\left( \frac{3 + 2S}{\epsilon}  \right)^{d\times r + r\times d' \times P}}$ and therefore $\ln \left( \frac{1}{\delta} \right) \leq \ln(1/\delta) + \log(\frac{3+2S}{\epsilon})*(dr + rd' P) $. Having applied the union bound over the $\epsilon$ cover over tuples $(\mathbf{B}, \mathbf{A}^{(1)}, \cdots, \mathbf{A}^{(P)})$ in the previous discussion and plugging in the definition of $\delta$ in the display above yields that with probability at least $1-\delta$ 
\begin{align*}
  \sum_{p\in[P]}  I(p) &\leq \frac{1}{2}  \sum_{(n', h') \leq (n,H) } \sum_{p\in[P]} \left\| \left(\widetilde{\mathbf{B}}_n \widetilde{\mathbf{A}}_n^{(p)} - \mathbf{B}_\star \mathbf{A}_\star^{(p)} \right)^\top \phi_{n', h}^{(p)}\right\|^2 + 1 + L_\phi S + \frac{b^2}{2R^2} +\\
    &\qquad  (2.5*4R^2 + 0.41 b) \Big( 1.4 \ln \ln \left(2 \left(nHP\right)\right) + \ln{5.2} + \ln\frac{1}{\delta} + \\
    &\qquad  (dr + rd'P)\left( \ln(3+2S) + \ln{nHP} + \ln{2RL_\phi} \right) \Big) \\
\end{align*}
Let the radius $\beta'_{nH}(\delta)$ be defined as,
\begin{align*}
\beta'_{nH}(\delta) &= 1 + L_\phi S + \frac{b^2}{2R^2} +  (2.5*4R^2 + 0.41 b) \Big( 1.4 \ln \ln \left(2 \left(nHP\right)\right) + \ln{5.2} + \ln\frac{1}{\delta} +  \\
&  \quad  (dr + rd'P)\left( \ln(3+2S) + \ln{nHP} + \ln{2RL_\phi} \right) \Big)
\end{align*}
Where $b = 2R d' S L_{\psi}$. Combining this upper bound with Equation~\ref{equation::upper_bound_regularized} we obtain that with probability at least $1-\delta$:

\begin{align*}
     \sum_{p \in [P]} \lambda \|  \widetilde{\mathbf{A}}_n^{(p)} \|_F^2 +  \frac{1}{2} \sum_{n'<n, h\leq H}  \| \left(\mathbf{B}_\star \mathbf{A}_\star^{(p)} - \widetilde{\mathbf{B}}_n \widetilde{\mathbf{A}}_n^{(p)} \right)^\top \phi_{n', h}^{(p)}\|^2  \leq     \beta'_{nH}(\delta) + \sum_{p \in [P]} \lambda \| \widetilde{\mathbf{A}}_\star^{(p)}\|_F^2 
\end{align*}

By definition 

$$\sum_{n'<n, h\leq H}  \left\| \left(\mathbf{B}_\star \mathbf{A}_\star^{(p)} - \widetilde{\mathbf{B}}_n \widetilde{\mathbf{A}}_n^{(p)} \right)^\top \phi_{n', h}^{(p)}\right\|^2  = \left\| \left(\Sigma^{(p)}_n\right)^{1/2} \left(\mathbf{B}_\star \mathbf{A}_\star^{(p)} - \widetilde{\mathbf{B}}_n \widetilde{\mathbf{A}}_n^{(p)} \right) \right\|_F^2.$$

And therefore

\begin{equation*}
     \sum_{p \in [P]} \lambda \left\|  \widetilde{\mathbf{A}}_n^{(p)} \right\|_F^2 +   \frac{1}{2}  \left\| \left(\Sigma^{(p)}_n\right)^{1/2} \left(\mathbf{B}_\star \mathbf{A}_\star^{(p)} - \widetilde{\mathbf{B}}_n \widetilde{\mathbf{A}}_n^{(p)} \right) \right\|_F^2  \leq  \beta'_{nH}(\delta) +     \sum_{p \in [P]} \lambda \left\| \mathbf{A}_\star^{(p)}\right\|_F^2 
\end{equation*}

\end{proof}

We can derive a version of Lemma~\ref{lemma::recursive_bounding_Q} adapted to the Shared structure setting in this Section.

\subsection{Proof of Lemma~\ref{lemma::upper_bound_bellman_error}}\label{section::bound_bellman_error}

\sumoptimismhelper*

\begin{proof}

If Assumption~\ref{assumption::stronger_feature_regularity} holds,

\begin{align*}
   \sum_{p \in [P]} Q^{(p)}_{n,h}( s^{(p)}_{n,h}, a^{(p)}_{n,h}) &- \left( r(s^{(p)}_{n,h}, a^{(p)}_{n,h}) + \mathbb{P}^{(p)}(\cdot | s_{n,h}, a_{n,h})^\top V^{(p)}_{n,h+1}\right)  \\
    &\leq \sum_{p \in [P]} \left\|\left(\phi^{(p)}_{n,h}\right)^\top \left( \widebar{\M}^{(p)}_n - \M_\star^{(p)} \right) \right\|_2 \left\| \left(\Psi^{(p)}\right)^\top V^{(p)}_{n,h+1}\right\|_2 \\
    &\stackrel{(i)}{\leq} \sum_{p \in [P]} C_\psi \left\| V^{(p)}_{n,h+1}\right\|_\infty \left\| \left(\phi^{(p)}_{n,h}\right)^\top \left( \widebar{\M}^{(p)}_n - \M^{(p)}_\star \right) \right\|_2 \\
    &\stackrel{(ii)}{\leq} C_\psi H \sum_{p \in [P]}\left\| \left(\phi^{(p)}_{n,h}\right)^\top \left( \widebar{\M}^{(p)}_n - \M_\star^{(p)} \right) \right\|_2 \\
    &\stackrel{(iii)}{\leq} C_\psi H \sum_{p \in [P]}\left(\left\| \left(\phi^{(p)}_{n,h}\right)^\top \left( \widebar{\M}^{(p)}_n - \widetilde{\M}^{(p)}_n \right) \right\|_2 + \left\| \left(\phi^{(p)}_{n,h}\right)^\top \left( \widetilde{\M}^{(p)}_n - \M^{(p)}_\star \right) \right\|_2  \right) \\
\end{align*}

Inequality $(i)$ follows by Assumption~\ref{assumption::feature_regularity}, and inequality $(ii)$ holds because the range of $V_{n,h+1}$ is bounded by $H$. Finally $(iii)$ is a consequence of the triangle inequality. We focus now on bunding the right hand side of the last inequality,

\begin{align*}
    \sum_{p \in [P]}\left(\left\| \left(\phi^{(p)}_{n,h}\right)^\top \left( \widebar{\M}^{(p)}_n - \widetilde{\M}^{(p)}_n \right) \right\|_2 + \left\| \left(\phi^{(p)}_{n,h}\right)^\top \left( \widetilde{\M}^{(p)}_n - \M^{(p)}_\star \right) \right\|_2  \right)  \leq \qquad \qquad \qquad \qquad  \qquad \qquad \\
    \qquad \qquad \qquad \sum_{p \in [P]} \| \phi_{n,h}^{(p)}\|_{\left(\Sigma_n^{(p)}\right)^{-1}} \left(    \left\|  \left(\Sigma^{(p)}_n\right)^{1/2} \left( \widebar{\M}^{(p)}_n - \widetilde{\M}^{(p)}_n \right)    \right \|_F  + \left\|\left(\Sigma^{(p)}_n\right)^{1/2} \left(\widetilde{\M}^{(p)}_n - \M^{(p)}_\star\right)\right\|_F  \right)  
\end{align*}
Let's bound each of the two summands. 
\begin{small}
\begin{align*}
    \sum_{p \in [P]} \| \phi_{n,h}^{(p)}\|_{\left(\Sigma_n^{(p)}\right)^{-1}}  \left\|  \left(\Sigma^{(p)}_n\right)^{1/2} \left( \widebar{\M}^{(p)}_n - \widetilde{\M}^{(p)}_n \right)     \right\|_F  &\leq 
    \sqrt{ \sum_{p \in P} \| \phi_{n,h}^{(p)}\|^2_{\left(\Sigma_n^{(p)}\right)^{-1}}    } \sqrt{ \sum_{p \in [P]} \left\|  \left(\Sigma^{(p)}_n\right)^{1/2} \left( \widebar{\M}^{(p)}_n - \widetilde{\M}^{(p)}_n \right)     \right\|^2_F } \\
    &\leq 
    \sqrt{ \sum_{p \in P} \| \phi_{n,h}^{(p)}\|^2_{\left(\Sigma_n^{(p)}\right)^{-1}}    } \sqrt{ \sum_{p \in [P]} \left\|  \left(\Sigma^{(p)}_n\right)^{1/2} \left( \widebar{\M}^{(p)}_n - \widetilde{\M}^{(p)}_n \right)     \right\|^2_F } \\
    &\leq \sqrt{  \sum_{p \in P} \| \phi_{n,h}^{(p)}\|^2_{\left(\Sigma_n^{(p)}\right)^{-1}} }\sqrt{ \gamma_n(\delta) }  
\end{align*}
\end{small}
The first inequality holds by Holder and the last inequality holds because $  \widebar{\mathbf{M}}_n^{(1)}, \cdots, \widebar{\mathbf{M}}_n^{(P)} $ is in $\widetilde{\mathbf{U}}_n^F$. Similarly,
\begin{small}
\begin{align*}
      \sum_{p \in [P]} \| \phi_{n,h}^{(p)}\|_{\left(\Sigma_n^{(p)}\right)^{-1}}  \|\left(\Sigma^{(p)}_n\right)^{1/2} \left(\widetilde{\M}^{(p)}_n - \M^{(p)}_\star\right)\|_F  
     &\leq\sqrt{\sum_{p \in [P]} \| \phi_{n,h}^{(p)}\|^2_{\left(\Sigma_n^{(p)}\right)^{-1}}}  \sqrt{ \sum_{p \in [P]} \left\| \left(\Sigma^{(p)}_n\right)^{1/2} \left(\widetilde{\M}^{(p)}_n - \M^{(p)}_\star\right)\right\|^2_F } \\
     &\leq\sqrt{ \sum_{p \in [P]} \| \phi_{n,h}^{(p)}\|^2_{\left(\Sigma_n^{(p)}\right)^{-1}}}  \sqrt{ \sum_{p \in [P]} \left\| \left(\Sigma^{(p)}_n\right)^{1/2} \left(\widetilde{\M}^{(p)}_n - \M^{(p)}_\star\right)\right\|^2_F } \\
     &\leq    \sqrt{ \sum_{p \in [P]} \| \phi_{n,h}^{(p)}\|^2_{\left(\Sigma_n^{(p)}\right)^{-1}}} \sqrt{\gamma_{n}(\delta)}  \qquad \qquad \qquad  \qquad \qquad \qquad 
\end{align*}
\end{small}
The first inequality holds by Holder and the last inequality holds because with high probability $ \mathbf{M}_\star^{(1)}, \cdots, \mathbf{M}_\star^{(P)} $ is in $\widetilde{\mathbf{U}}_n^F(\delta)$.

The result follows.

\end{proof}

\subsection{Proof of Theorem~\ref{theorem::regret_guarantee_shared}}

\maintheoremshared*

\begin{proof}

Recall the shared regret equals $ R_P(NH) = \sum_{n=1}^N \sum_{p \in [P] } V_1^{\pi_\star^{(p)}}(s_{n,1}^{(p)}) - V_1^{\pi_n^{(p)} }(s^{(p)}_{n,1})$. The optimism property of Lemma~\ref{lemma::shared_optimism} implies that,

\begin{align*}
    R_P(NH) &\leq \sum_{n=1}^N \sum_{p \in [P] } V_{n,1}^{(p)}(s_{n,1}^{(p)}) - V_1^{\pi_n^{(p)} }(s^{(p)}_{n,1}) \\
    &\stackrel{(i)}{\leq} \sum_{n=1}^N \Big(  2C_\psi H  \sqrt{\gamma_{n}(\delta)}\sqrt{ \sum_{p \in [P]} \| \phi_{n,h}^{(p)}\|^2_{\left(\Sigma_n^{(p)}\right)^{-1}}}  + \sum_{p \in [P]} \mathbb{P}^{(p)}( \cdot  | s_{n,h}^{(p)}, a_{n,h}^{(p)}   )^\top \left(  V_{n,2}^{(p)} - V_{2}^{\pi^{(p)}_n}       \right)    \Big)\\
    &= \sum_{n=1}^N \Big(  2C_\psi H  \sqrt{\gamma_{n}(\delta)}\sqrt{ \sum_{p \in [P]} \| \phi_{n,h}^{(p)}\|^2_{\left(\Sigma_n^{(p)}\right)^{-1}} }  + \\
    &\quad \sum_{p \in [P]} \underbrace{\mathbb{P}^{(p)}( \cdot  | s_{n,h}^{(p)}, a_{n,h}^{(p)}   )^\top \left(  V_{n,2}^{(p)} - V_{2}^{\pi^{(p)}_n}       \right)   - \left( V_{n,2}(s_{n,2}^{(p)})  - V_2^{\pi_{n}^{(p)}}(s_{n,2}^{(p)}) \right)}_{\delta^{(p)}_{n,2}} + \\
    &\quad \sum_{p \in [P]} V_{n,2}(s_{n,2}^{(p)})  - V_2^{\pi_{n}^{(p)}}(s_{n,2}^{(p)}) \Big)
\end{align*}
Where $(i)$ is a consequence of Lemma~\ref{lemma::upper_bound_bellman_error}. A recursive use of this decomposition yields,
\begin{align*}
     R_P(NH)  &\leq \sum_{n=1}^N  \sum_{h=1}^H 2C_\psi H  \sqrt{\gamma_{n}(\delta)}\sqrt{\sum_{p \in [P]} \| \phi_{n,h}^{(p)}\|^2_{\left(\Sigma_n^{(p)}\right)^{-1}}} + \sum_{p \in [P]} \delta^{(p)}_{n,h} \\
     &\leq \underbrace{ \sum_{n=1}^N  \sum_{h=1}^H 2C_\psi H  \sqrt{\gamma_{N}(\delta)}\sqrt{\sum_{p \in [P]} \| \phi_{n,h}^{(p)}\|^2_{\left(\Sigma_n^{(p)}\right)^{-1}}}}_{\mathbf{I}} + \sum_{p \in [P]} \delta^{(p)}_{n,h}
\end{align*}

Where the last inequality holds because $\gamma_n$ is increasing in $n$.

 The sum  $\sum_{n=1}^N \sum_{p \in [P]} \sum_{h=1}^H \delta^{(p)}_{n,h}$ can easily be bounded by invoking Lemma~\ref{lemma::matingale_concentration_anytime} by observing that $| \delta^{(p)}_{n,h} | \leq 4H$ for all $n,h, p$. Thus, with probability at least $1-\delta$ for all $n$,

\begin{equation*}
    \sum_{n=1}^N \sum_{h=1}^H \sum_{p \in [P]} \delta^{(p)}_{n,h} \leq 8H \sqrt{NHP\log\left( \frac{6\log NH}{\delta}\right)  }.
 \end{equation*}

By Corollary~\ref{corollary::n_inverse_to_n_h_inverse_multitask},

\begin{equation*}
 \sum_{n=1}^N \sum_{h=1}^H   \sqrt{ \sum_{p \in [P]} \| \phi_{n,h}^{(p)}\|^2_{\left(\Sigma_n^{(p)}\right)^{-1}}}  \leq \sum_{n=1}^N \sum_{h=1}^H   2\sqrt{ \sum_{p \in [P]} \| \phi_{n,h}^{(p)}\|^2_{\left(\Sigma_{n,h}^{(p)}\right)^{-1}}} + \frac{2L_\phi HdP}{\sqrt{\lambda}} \log\left( 1+\frac{  N H L_\phi^2 }{\lambda d}\right).
\end{equation*}

By Cauchy-Schwartz,

\begin{equation*}
 \sum_{n=1}^N \sum_{h=1}^H   \sqrt{ \sum_{p \in [P]} \| \phi_{n,h}^{(p)}\|^2_{\left(\Sigma_{n,h}^{(p)}\right)^{-1}}}  \leq \sqrt{NH \sum_{n=1}^N \sum_{h=1}^H \sum_{p \in [P]}  \| \phi_{n,h}^{(p)}\|^2_{\left(\Sigma_{n,h}^{(p)}\right)^{-1}}}  
\end{equation*}

By Lemma~\ref{lemma:det_lemma} (setting $b = \frac{L_{\phi}}{\sqrt{\lambda}}$),

\begin{equation*}
   \sum_{n=1}^N \sum_{h=1}^H \sum_{p \in [P]} 2\| \phi^{(p)}_{n,h}\|^2_{\Sigma_{n,h}^{-1}} \leq P \left(1+ \frac{L_\phi^2}{\sqrt{\lambda}}\right)d \log\left( 1 + \frac{NHL_\phi^2}{\lambda d}\right) 
\end{equation*}

Therefore 
\begin{align*}
     \sum_{n=1}^N \sum_{h=1}^H   \sqrt{ \sum_{p \in [P]} \| \phi_{n,h}^{(p)}\|^2_{\left(\Sigma_n^{(p)}\right)^{-1}}} &\leq \sqrt{NH P \left(1+\frac{L_\phi^2}{\sqrt{\lambda}}\right)d \log\left( 1 + \frac{NHL_\phi^2}{\lambda d}\right)  } + \\
     &\quad \frac{2L_\phi HdP}{\sqrt{\lambda}} \log\left( 1+\frac{  N H L_\phi^2 }{\lambda d}\right).
\end{align*}

We then conclude that with probability at least $1-\delta$, whenever $\mathcal{E}'$ holds,

\begin{align*}
    R_P(NH) &\leq  H \sqrt{NHP\log\left( \frac{6\log NH}{\delta}\right)  } + \\
    &\quad 2C_\psi H\sqrt{ \gamma_N(\delta)} \sqrt{NH P \left(1+ \frac{L_\phi^2}{\sqrt{\lambda}}\right)d \log\left( 1 + \frac{NHL_\phi^2}{\lambda d}\right)  } +\\
    &\quad \frac{4C_\psi L_\phi H^2dP\sqrt{\gamma_N(\delta)}}{\sqrt{\lambda}} \log\left( 1+\frac{  N H L_\phi^2 }{\lambda d}\right).
\end{align*}

\end{proof}

\section{Supporting Technical Lemmas}

\begin{lemma}\label{lemma::matingale_concentration_anytime}
Let $\{X_t\}_{t=1}^\infty$ be a martingale difference sequence with $| X_n | \leq \zeta$ and $\E \left[X_n\right] = 0$,  and let $\delta \in (0,1]$. Then with probability $1-\delta$ for all $N \in \mathbb{N}$
\begin{equation*}
    \sum_{t=1}^N X_n \leq 2\zeta \sqrt{N \ln \left(\frac{6\ln N}{\delta}\right) }.
\end{equation*}
\end{lemma}
\begin{proof}Observe that $\frac{\left|x_n\right|}{\zeta}\le 1$. By invoking a time-uniform Hoeffding-style concentration inequality \citep[][Equation~(11)]{howard2020time} we find that
\begin{align*}
    \Pr\left[\forall \;  n \in \N\;:\; \sum_{t=1}^{N} \frac{X_n}{\zeta} \le 1.7\sqrt{N \left(\log\log(T)+0.72\log\left(\frac{5.2}{\delta}\right)\right)} \right]\ge 1-\delta.
\end{align*}
Rounding up the constants for the sake of simplicity we get
\begin{align*}
      \Pr\left[\forall \; t \in \N\;:\; \sum_{t=1}^{T} X_t \le 2\zeta\sqrt{T \left(\log\left(\frac{6\log(T)}{\delta}\right)\right)} \right]\ge 1-\delta,  
\end{align*}
which establishes our claim.
\end{proof}

It can be shown that with high probability and for all $t$ simultaneously all $\widetilde{\M}_n$ lie in a vicinity of $\M$, for that we will make use of the following Theorem,

\begin{theorem}\label{thm:conf_ellipse_linucb}[Theorem 1 in \cite{abbasi2011improved}]
Let $\{\x_i\}_{i=1}^\infty \subset \mathbb{R}^{m}$ with $\| \x_i \| \leq L$ for all $i \in [n]$. And let $\{y_i\}_{i=1}^\infty$ be response random variables satisfying $y_\ell = \btheta_\star^\top \x_\ell + \eta_\ell$ where $\| \btheta\| \leq S$ such that $\mathbb{E}\left[ \eta_\ell  | \mathcal{F}_{\ell-1} \right] = 0$ and $\eta_\ell$ is $R-$subgaussian. Let $\widehat{\btheta}_t$ be the ridge regression estimator of $\btheta_\star$ with covariates $\{ \x_i \}_{i=1}^{t-1}$, responses $\{ y_i \}_{i=1}^{t-1}$, and regularizer $\lambda > 0$. For all $t \in \mathbb{N}$,
\begin{equation*}
    \norm{\hat{\mathbf{\theta}}_{t}-\thetastar}_{V_t} \leq \sqrt{\beta_t(\delta)} 
\end{equation*}
with probability at least $1-\delta$. Where $\beta_t(\delta) = R\sqrt{ m \log\left( \frac{1+ tL^2/\lambda}{\delta}\right) } + \sqrt{\lambda}S$ and $V_t = \lambda \mathbb{I} + \sum_{t=1}^T \x_i\x_i^\top$.
\end{theorem}

\begin{lemma}[Uniform empirical Bernstein bound]
\label{lem:uniform_emp_bernstein}
In the terminology of \citet{howard2018uniform}, let $S_t = \sum_{i=1}^t Y_i$ be a sub-$\psi_P$ process with parameter $c > 0$ and variance process $W_t$. Then with probability at least $1 - \delta$ for all $t \in \mathbb{N}$
\begin{align*}
    S_t &\leq  1.44 \sqrt{\max(W_t , m) \left( 1.4 \ln \ln \left(2 \left(\max\left(\frac{W_t}{m} , 1 \right)\right)\right) + \ln \frac{5.2}{\delta}\right)}\\
   & \qquad + 0.41 c  \left( 1.4 \ln \ln \left(2 \left(\max\left(\frac{W_t}{m} , 1\right)\right)\right) + \ln \frac{5.2}{\delta}\right)
\end{align*}
where $m > 0$ is arbitrary but fixed.
\end{lemma}

The following lemma relates the values of quadratic forms of Positive Semi-definite Matrices with the ratio of their determinants.
\begin{lemma}\label{lemma::supporting_lin_alg_result}
If $\mathbf{B} \succeq \mathbf{C} \succ \mathbf{0}$ be $d\times d$ dimensional matrices then,
\begin{equation*}
    \sup_{\mathbf{x}\neq 0}\frac{\mathbf{x}^\top \mathbf{B} \mathbf{x} }{ \mathbf{x}^\top \mathbf{C} \mathbf{x} } \leq \frac{\mathrm{det}( \mathbf{B}) }{\mathrm{det}( \mathbf{C})}.
\end{equation*}
\end{lemma}
\begin{proof}
Given any $\mathbf{y} \in \mathbb{R}^d$ let $\mathbf{x} = \mathbf{C}^{-1/2}\mathbf{y}$. Then
\begin{align*}
    \sup_{\mathbf{x} \neq 0} \frac{\mathbf{x}^\top \mathbf{B} \mathbf{x}}{\mathbf{x}^\top \mathbf{C} \mathbf{x}}=    \sup_{\mathbf{y} \neq 0} \frac{\mathbf{y}^\top \mathbf{C}^{-1/2}\mathbf{B}\mathbf{C}^{-1/2} \mathbf{y}}{\lv \mathbf{y} \rv_2^2} = \norm{\mathbf{C}^{-1/2} \mathbf{B} \mathbf{C}^{-1/2}}_{op}
\end{align*}
by the definition of the operator norm.  Recall that by assumption $\mathbf{B}-\mathbf{C}\succeq 0$ therefore $\mathbf{C}^{-1/2}\mathbf{B}\mathbf{C}^{-1/2}-\mathbf{I}\succeq 0$, and hence all the eigenvalues of $\mathbf{C}^{-1/2} \mathbf{B} \mathbf{C}^{-1/2}$ are at least $1$. Thus 
\begin{align*}
     \sup_{\mathbf{x} \neq 0} \frac{\mathbf{x}^\top \mathbf{B} \mathbf{x}}{\mathbf{x}^\top \mathbf{C} \mathbf{x}} \le \norm{\mathbf{C}^{-1/2} \mathbf{B} \mathbf{C}^{-1/2}}_{op}\le \det(\mathbf{C}^{-1/2} \mathbf{B} \mathbf{C}^{-1/2}) = \frac{\mathrm{det}( \mathbf{B}) }{\mathrm{det}( \mathbf{C})},
\end{align*}
where the last equality follows since $\frac{\det(\mathbf{B})}{\det(\mathbf{C})} = \det(\mathbf{C}^{-1/2})\det(\mathbf{B})\det(\mathbf{C}^{-1/2})= \det(\mathbf{C}^{-1/2} \mathbf{B} \mathbf{C}^{-1/2})$. This completes the proof.
\end{proof}

\end{document}